\documentclass{article}

\usepackage{microtype}
\usepackage{graphicx}
\usepackage{subfigure}
\usepackage{booktabs} 

\usepackage{hyperref}
\usepackage{enumitem}
\usepackage{bibunits}

\setitemize{noitemsep,topsep=0pt,parsep=0pt,partopsep=0pt}
\setenumerate{noitemsep,topsep=0pt,parsep=0pt,partopsep=0pt}

\usepackage[accepted]{icml2020}
\usepackage[english]{babel}
\usepackage[dvipsnames]{xcolor}
\usepackage[T1]{fontenc}    
\usepackage{hyperref}       
\usepackage{url}            
\usepackage{booktabs}       
\usepackage{amsfonts}       
\usepackage{nicefrac}       
\usepackage{microtype}      
\usepackage{natbib}
\usepackage{amsthm}

\usepackage{graphicx}
\usepackage{mathtools}
\usepackage{footnote}
\usepackage{float}
\usepackage{xspace}
\usepackage{multirow}
\usepackage{wrapfig}
\usepackage{framed}
\usepackage{makecell}
\usepackage[outdir=./]{epstopdf}

\newcommand{\calA}{{\mathcal{A}}}

\newcommand{\Reg}{\text{\rm Reg}}

\newcommand{\one}{\boldsymbol{1}}

\DeclareMathOperator*{\argmin}{argmin}
\DeclareMathOperator*{\argmax}{argmax}

\newcommand{\field}[1]{\mathbb{#1}}

\newcommand{\E}{\field{E}}
\newcommand{\hatL}{\widehat{L}}

\newcommand{\cb}{\text{CB}}

\newcommand{\trace}[1]{\text{trace}\left({#1}\right)}

\newcommand{\up}{\alpha}
\newcommand{\down}{\beta}

\newcommand{\term}{\textbf{term}}
\newcommand{\Var}{\mathbb{V}}
\newcommand{\Cov}{\text{Cov}}

\newtheorem{lemma}{Lemma}
\newtheorem{theorem}{Theorem}
\newtheorem*{theorem*}{Theorem}

\newtheorem{definition}{Definition}

\newenvironment{rtheorem}[3][]{

\bigskip

\noindent \ifthenelse{\equal{#1}{}}{\bf #2 #3}{\bf #2 #3 (#1)}
\begin{it}
}{\end{it}}

\newcommand{\order}{\ensuremath{\mathcal{O}}}
\newcommand{\otil}{\ensuremath{\widetilde{\mathcal{O}}}}

\usepackage[vlined,linesnumbered,ruled, algo2e]{algorithm2e}
\usepackage{bbm}
\usepackage{amsfonts}
\usepackage{amsmath}
\usepackage{mathtools}
\usepackage{amsopn}
\usepackage{amssymb}
\usepackage{graphicx}

\usepackage{blindtext}

\theoremstyle{definition}
\newtheorem{example}{Example}
\allowdisplaybreaks

\usepackage{amsmath} 
\newcommand{\dist}{\ensuremath{\lambda}}

\icmltitlerunning{Federated Residual Learning}

\newcommand{\alg}{FedRes\xspace}
\newcommand{\alglong}{Federated Residual Learning\xspace}

\begin{document}
\twocolumn[
\icmltitle{Federated Residual Learning}




\begin{icmlauthorlist}
\icmlauthor{Alekh Agarwal}{Alekh Agarwal}
\icmlauthor{John Langford}{John Langford}
\icmlauthor{Chen-Yu Wei}{Chen-Yu Wei}

\end{icmlauthorlist}
\icmlaffiliation{Alekh Agarwal}{Microsoft Research, Redmond}
\icmlaffiliation{John Langford}{Microsoft Research, New York City}
\icmlaffiliation{Chen-Yu Wei}{University of Southern California}
\icmlcorrespondingauthor{Chen-Yu Wei}{chenyu.wei@usc.edu}
\icmlkeywords{Federated learning, distributed learning}

\vskip 0.3in
]



\printAffiliationsAndNotice{This work was done when Chen-Yu Wei was an intern at Microsoft Research, Redmond. }  

\begin{abstract}
We study a new form of federated learning where the clients train personalized local models and make predictions jointly with the server-side shared model. Using this new federated learning framework, the complexity of the central shared model can be minimized
while still gaining all the performance benefits that joint
training provides. Our framework is robust to data heterogeneity, addressing the slow convergence problem traditional federated learning methods face when the data is non-i.i.d. across clients. We test the theory empirically and find substantial performance gains over baselines.
\end{abstract}

\section{Introduction}
\label{section: introduction}
In federated learning~\cite{mcmahan2017communication, smith2017federated, DBLP:journals/corr/abs-1802-07876}, the training samples are acquired from a host of clients. The goal is to learn a significantly more accurate model than each client could achieve using just the locally available data.  Most prior work considered learning a single centralized model by incorporating the samples from all the clients. While this scheme indeed provides the benefits of joint training, increasing the overall data efficiency, its performance suffers when the clients have different data distributions \cite{li2019federated}. In this paper, we provide a solution that enables federated learning to work well in such environments, while preserving all the desirable properties.


To illustrate the key challenges of our setting, we adopt the problem of content recommendation as a main motivating example throughout the paper. In this setting, each client is typically a computer or a mobile device, associated with a user.  The goal of learning is to improve the user's engagement with the presented content, measured via metrics such as click-through rate or dwell time. Depending on the approach, the learning task might involve predicting the values of these metrics, and use them to guide the recommendation decision. Some salient aspects of this setting are:
\begin{enumerate}
\item Different users have different preferences, so personalized model is needed.
\item The data samples collected from each user are not enough to train a powerful personalized model.
\item Incorporating all personalization in a centralized model can result in a huge model size, making it intractable.
\end{enumerate}

We address the above issues by proposing a \emph{model separation} approach, a new form of federated learning. Specifically, we consider the scenario where the server of the system maintains a \emph{global model} that is shared across all clients, and each client maintains its own personalized \emph{local model}.  For a certain client's prediction task, the prediction is jointly made by the global model and the local model. As a simple example, we can let the final prediction value to be the sum of the prediction values given by the global model and the local model.

For this setting, we develop novel federated learning algorithms. Since making the prediction on an example requires the predictions of both the global and the local models, they are effectively learn against the \emph{residuals} from the other one. Therefore, we name our framework and algorithms \alglong, or simply \alg.

This new framework has several desirable properties that make it suitable for large-scale deployment. First, the clients have freedom to design their own local models and the \emph{local features} that the local models are trained on. This allows devices of different hardware complexity to join the federated system with low cost. Second, in a version of our algorithm (i.e., the SGD-variant introduced in Section~\ref{subsec: sgd}), all information about the local model and the local features that the client uses to train the local model can be summarized as \emph{residuals} for the server. Since the residuals can usually be represented by a few bits for each data sample, the communication between the clients and the server can be rather efficient. Furthermore, since the client does not need to reveal the design of the local model and the local features it uses, the system largely preserves privacy.

 To model the real-world scenario, we incorporate the \emph{delay} between the server and the clients into our algorithm design and analysis, making our algorithm robust to delay. This is inspired by prior works on delayed feedback stochastic optimization~\cite{zinkevich2009slow, agarwal2011distributed, duchi2011dual,
  dekel2012optimal}, but requires new insights because our problem is complicated by the federated structure. We derive regret bounds for our algorithms, exhibiting improvements over purely global and local learning schemes, and showing its robustness to delays.  Our algorithms and analysis nicely work with mini-batches, which we show in Section~\ref{section: mini-batch}.

  Empirically, we evaluate the algorithm across a number of datasets. We demonstrate the efficacy of our algorithm over natural baselines as well as showing its robustness to delays and data heterogeneity.  Figure~\ref{fig:50 worker overall} provides an example, showing that \alglong yields superior performance over baselines operating with the same constraints.
\begin{figure}[t]
    \centering
    \includegraphics[width=0.5\textwidth,trim={0 0pt 0 0},clip]{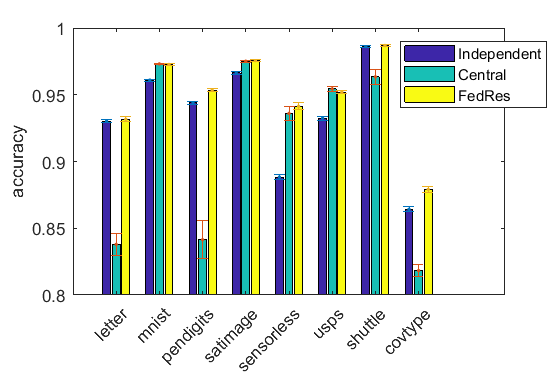}
    \caption{Accuracy of Independent, Central, and \alg approaches with $P=50$ clients in the absence of communication delays. The \alg approach is always nearly the best and substantially superior to alternatives in some cases.}
    \label{fig:50 worker overall}
\end{figure}

We note that our approach here is very basic, essentially a
modification of empirical risk minimization and gradient descent.  As such, it has general
applicability to many kinds of models---handwriting recognition,
reinforcement learning, and machine translation are all possibilities,
for example.

\subsection{Related work}
Federated learning has become a popular topic in machine learning. As proposed, the main focus is on communication efficiency~\cite{mcmahan2017communication}, with a global shared model in the federated learning system. There are also works dealing with the heterogeneity of the data distribution in federated systems \cite{smith2017federated, DBLP:journals/corr/abs-1802-07876, li2018federated,  mohri2019agnostic, karimireddy2019scaffold, jiang2019improving}. However, a fundamental difference between our work and theirs is that their global models and the local models still operate in the same parameter space, while our framework provides more flexibility in the design of local models, as we will see in Section~\ref{section: setting}.

Several papers have addressed stochastic optimization with
delayed feedback
\cite{zinkevich2009slow, agarwal2011distributed, duchi2011dual,
  dekel2012optimal} with different approaches. However, they all concluded that the asymptotic performance of stochastic optimization is not affected by the delay in feedback, provided that the amount of delay is bounded, and the objective function is smooth. Inspired by them, we extend their results to the more challenging federated setting, and draw similar conclusions.

Regarding how to reduce the complexity of a centralized model, the work of~\cite{spam09} proposed to use the feature hashing approach for spam filtering. Although they demonstrated dramatic compressions, there are many other applications where feature hashing
may harm performance.

\section{Problem Setting}
\label{section: setting}


We consider an online learning scenario in a federated learning system which consists of one server and $P$ clients. At any time $t$, the server keeps a \emph{global model}, which can be parameterized by a vector $w^g_t\in \mathbb{R}^d$ and each client $i$ keeps a \emph{local model}, parameterized by $w_{i,t}\in\mathbb{R}^{d_i}$. At each round $t$, client $i$ observes a feature vector $x_{i,t}=(x_{i,t}^g, x_{i,t}^l)$, where $x_{i,t}^g$ consists of \emph{global features}, and $x_{i,t}^l$ consists of \emph{local features}. The goal of client $i$ is to predict the label jointly with the global and the local models. More precisely, the global model gives a value $f(x_{i,t}^g; w_t^g)$ using global features; the local model gives another value $f(x_{i,t}^l; w_{i,t})$ using local features. They jointly incur a loss of
\begin{align*}
    \ell_{i,t}(w_t^g, w_{i,t})\triangleq \ell\Big(y_{i,t}, f(x_{i,t}^g; w_t^g), f(x_{i,t}^l; w_{i,t})  \Big),
\end{align*}
where $y_{i,t}$ is the true label, and $\ell$ is a loss function that reflects the accuracy of the joint prediction. An example of $\ell$ is the squared regression loss: $\ell(y,\hat{y}^g, \hat{y}^l) = (y-\hat{y}^g-\hat{y}^l)^2$.  

We also use $\dist$ to denote a set of weights over the clients, where $\dist_i \geq 0$ for all $i=1,2,\ldots,P$. The overall goal is to have low regret against the optimal joint global and local models. The (average) regret is defined as
\begin{equation}
\Reg = \scalebox{1.0}{$\displaystyle\sup_{u^g, u_{i}}\frac{1}{PT}\sum_{i=1}^P \dist_i \sum_{t=1}^T \left(\ell_{i,t}(w^g_t, w_{i,t}) - \ell_{i,t}(u^g, u_{i})\right).$}
\label{eqn:objective}
\end{equation}
It might appear that the model requires all clients to see the same number of examples as we draw a loss function $\ell_{i, t}$ for each client $i$ on every round $t$. We can easily circumvent this by setting the loss function to be identically $0$ if no data was observed on some round. Formally, if $N_i$ non-zero samples are observed at the client $i$, then setting $\lambda_i = T/N_i$ turns the objective into a sum of the average losses incurred at each client. For simplicity, in the later text, we all assume $\lambda_i=1$.

Below we give more concrete examples for our system.
\begin{example}[Linear regression]
    In this case, we define $\ell_{i,t}(w^g, w_i) = \left( y_{i,t} - w^{g\top} x^g_{i,t} - w_i^{\top}x^l_{i,t} \right)^2$ for some sample $(x_{i,t}^g, x_{i,t}^l, y_{i,t})\in\mathbb{R}^{d}\times \mathbb{R}^{d_i}\times \mathbb{R}$. Here, $y_{i,t}\in\mathbb{R}$ is the label; $x^g_{i,t}$ and $x^l_{i,t}$ are the features used by the global and local models respectively. Note that $x_{i,t}^g$ and $x_{i,t}^l$ can be identical, but we allow separate feature spaces for additional modeling flexibility.
\label{ex:lin_reg}
\end{example}

Typical works in federated learning focus on learning a good global model $w^g$ by minimizing the loss across all the clients. While this is desirable if the losses are drawn from an identical distribution across all the clients, it can fail to accurately predict at any client when they differ in a meaningful manner from each other. We now consider a further specialization of the example above to highlight the benefits of using a local model.

\begin{example}[Need for local models]
    In the setting of Example~\ref{ex:lin_reg} assume further that there exist vectors $u^g, \{u_i\}_{i=1}^P$ such that $y_{i,t} = u^g\cdot x_{i,t} + u_{i}\cdot x_{i,t}$ for all $i=1,\ldots,P$ and $t=1,\ldots,T$ where local and global features are identical. Assume $P$ is an even number and there is a vector $v$ such that $u_i = v$ for $i \leq P/2$ and $u_i = -v$ for $i > P/2$. The distribution of the covariates $x_{i,t}$ is identical across rounds and clients.  As $T$ becomes large, the optimal solution for our objective~\eqref{eqn:objective} coincides with the underlying parameters which generated the data. If we instead consider purely global training which would find $\min_w \sum_{i=1}^P \sum_{t=1}^T (y_{i,t} - w\cdot x_{i,t})^2$, then the solution of $w$ approaches $u^g$ as $T$ increases. However, when the model has converged, the clients still suffer a loss of $(v\cdot x_{i,t})^2$ for each sample. Thus, each client ends up with inaccurate predictions despite using a sufficiently expressive model.
\label{ex:lin_local}
\end{example}

In this work, we take into consideration the communication delay between the clients and the server. At each round, each client can upload data samples to the server, and/or fetch global models to the client side. We assume that at time $t$, client $i$ is able to fetch an outdated global model that is constructed at time $t-\down_i$, where $\down_i$ is the \emph{downlink delay} for client $i$. On the other hand, we assume that the data examples sent at time $t$ by client $i$ are received by the server at time $t+\up_i$, where $\up_i$ is the \emph{uplink delay} of client $i$. The \emph{round-trip delay} is denoted as $\tau_i=\up_i+\down_i$, and we assume $\tau_i\leq \tau$ for all clients $i$.

\paragraph{More notations and assumptions. }
For a random vector $v$, we use $\Var[v]$ to denote $\E[\|v-\E[v]\|^2]=\trace{\Cov[v]}$.
Denote the gradient of the losses with respect to global parameters and local parameters by $\nabla^g \ell_{i,t}(w^g, w_i)\triangleq \nabla_{w^g}\ell_{i,t}(w^g, w_i)$ and $\nabla^l\ell_{i,t}(w^g, w_i)\triangleq \nabla_{w_i}\ell_{i,t}(w^g, w_i)$. $\nabla \ell_{i,t}(w^g, w_i)$ denotes $\nabla_{(w^g, w_i)} \ell_{i,t}(w^g, w_i)$. For the loss function, we make the following assumptions for any pair $(w^g, w_i)$ such that $\|w^g\|, \|w_i\| \leq D$:
\begin{itemize}
    \item The value of the loss $\ell_{i,t}(w^g, w_i)$ lies in $[0,1]$.
    \item The losses are convex and $\gamma$-smooth jointly in both parameters. A function $f$ is $\gamma$-smooth if for all $a,b$
    \begin{align*}
        f(a) - f(b) \leq \nabla f(b)\cdot (a-b) + \frac{\gamma}{2}\|a-b\|^2.
    \end{align*}
    \item The $\ell_2$-norm of the gradient of the loss $\|\nabla \ell_{i,t}(w^g, w_i)\|$ is upper bounded by $G$.\footnote{Smoothness of $\ell$ implies that gradients exist almost everywhere so that we can avoid working with subgradients.}
\end{itemize}
We also assume that each client's data samples $(x_{i,t}^g, x_{i,t}^l,y_{i,t})$ are i.i.d. across time, but the distributions can differ across the different clients. We use $\Pi_D(v)\triangleq \argmin_{u: \|u\|\leq D}\|u-v\|$ to denote that projection operator onto a ball of radius $D$.


\section{Algorithms}
We extend two common statistical learning algorithms to our \alglong setting. One is the empirical risk minimization (ERM) approach that is fully general in that it can be coupled with any centralized loss minimization scheme, while the second is a stochastic gradient descent (SGD) approach which is a computationally attractive incremental approach for large-scale settings. We introduce them in Section~\ref{subsec: erm} and \ref{subsec: sgd} respectively.

\subsection{ERM-based approach}
\label{subsec: erm}
Empirical-risk minimization (ERM) is a simple and generic way of finding a good model given i.i.d. data samples. In the traditional centralized setting, the learner simply finds the model that minimizes the empirical loss on the previously observed data. We extend this algorithm to our    setting as follows (assuming $\up_i=\up$ and $\down_i=\down$ for all $i$): in each round, client $i$ fetches the newest global model $w^g_{t-\down}$, and then finds a local model $w_{i,t}$ which, together with $w^g_{t-\down}$, jointly minimize the empirical loss on all previously observed data of client $i$ (Algorithm~\ref{alg: ERM client}). On the server side, in each round, the server receives the newest data samples $z_{i,t-\up}$ and local models $w_{i,t-\up}$ from all clients, and then finds a global model $w_{t}^g$ that, together with all local models, jointly minimizes the total empirical loss across all the clients (Algorithm~\ref{alg: ERM server}).

\begin{algorithm}[t]
\caption{\alg.ERM.Client}
\label{alg: ERM client}
\For{$t=1, \ldots, T$}{
    Fetch the global model $w^g_{t-\down}$. \\
    Compute the local model:
    \begin{align}
        w_{i,t} = \argmin_{w: \|w\|\leq D}\left\{ \sum_{s=1}^{t-1} \ell_{i,s}(w_{t-\down}^g,  w)\right\}.   \label{eqn: client update rule erm}
    \end{align}
    Use the model pair $(w_{t-\down}^g, w_{i,t})$ to make preditions. \\
    Observe a new sample $z_{i,t} = (x_{i,t}^g, x_{i,t}^l, y_{i,t})$. \\
    Send $z_{i,t}$ and $w_{i,t}$ to the server.

}
\end{algorithm}

\begin{algorithm}[t]
\caption{\alg.ERM.Server}
\label{alg: ERM server}
\For{$t=1, \ldots, T$}{
    Receive $z_{i,t-\up}$ and $w_{i,t-\up}$ from all $i=1,\ldots, P$.
    \begin{align}
        w_{t+1}^g = \argmin_{w: \|w\|\leq D} \left\{  \sum_{i=1}^P\sum_{s=1}^{t-\up}\ell_{i,s}(w, w_{i,t-\up}) \right\} \label{eqn: server update rule erm}
    \end{align}
}
\end{algorithm}

Analyzing this algorithm is not as straightforward as in the centralized setting, because each client (server) is now facing a changing global (local) model, making the losses seen by the client (server) non-i.i.d. The algorithm is related to \emph{alternating minimization}, whose offline convergence property has been extensively studied in \cite{beck2015convergence}. Our analysis is inspired by \cite{beck2015convergence}, but further complicated because we deal with the online setting and consider the presence of delay. The following theorem gives a regret bound for this algorithm.

\begin{theorem}
    \label{theorem: main for erm}
    Suppose the variance of the loss $\Var[\ell_{i,t}(w^g, w_i)]$ is upper bounded by $\sigma^2$ for all $i,t$. Then
    \alg.ERM (Algorithm~\ref{alg: ERM client} and \ref{alg: ERM server}) guarantees
    \begin{align}
        &\E\left[ \frac{1}{PT}\sum_{i=1}^P \sum_{t=1}^T \ell_{i,t}\left(w^g_{t-\down}, w_{i,t}\right) - \ell_{i,t}\Big(w^g_*, w_{i,*}\Big)\right] \nonumber \\
        &=\widetilde{\order}\left(\sqrt{\frac{\left(d+\sum_{i=1}^P d_i\right)\sigma^2}{PT}} + \frac{\text{poly}(d,d_i,\gamma, D,\tau)}{T^{\frac{3}{4}}}\right).  \label{eqn: regret bound sgd}
    \end{align}
\end{theorem}
The exact form of the lower-order term can be found in the proof in the appendix.
To see the usefulness of the bound in Theorem~\ref{theorem: main for erm}, we assume that all local models have the same dimension $d_1=\cdots=d_P=d'$. Then the dominant term in the above bound can be written as
$
    \order\left( \sqrt{\frac{\left(\frac{d}{P} + d'\right)\sigma^2}{T}} \right).
$
Comparing this with the bound when each client indepdently performs ERM on the whole feature set:
$
    \order\left( \sqrt{\frac{\left(d + d'\right)\sigma^2}{T}} \right),
$
one can see that the complexity from the global features are amortized among the clients. On the other hand, the delay only affects a lower order term, adding relatively insignificant cost to the system.

One drawback of Algorithm~\alg.ERM is that the clients have to transmit both the data samples and the local model to the server. Also, to calculate a new local model, the clients have to apply the newly received global model $w_{t-\down}^g$ to all the previous samples (Eq.\eqref{eqn: client update rule erm}). This makes the system inefficient both in communication and computation. A natural fix to this problem is to let the clients and the server use the following update rules (\textit{cf. } \eqref{eqn: client update rule erm} and \eqref{eqn: server update rule erm}):
\begin{align}
     &w_{i,t+1} = \argmin_{w}\left\{ \sum_{s=1}^t \ell_{i,s}(w^g_{s-\down}, w) \right\}   \label{eqn: fp client} \\
     &w_{t+1}^g = \argmin_{w} \left\{  \sum_{i=1}^P\sum_{s=1}^{t-\up}\ell_{i,s}(w, w_{i,s}) \right\}  \label{eqn: fp server}
\end{align}
To execute this algorithm, the clients only need to send $\ell_{i,t}(\cdot, w_{i,t})$ to the server. Since $\ell_{i,t}(\cdot, w_{i,t})=\ell(y_{i,t}, f(x_{i,t}^g; \cdot), f(x_{i,t}^l; w_{i,t}))$, sending the triplet $(y_{i,t}, x_{i,t}^g, f(x_{i,t}^l; w_{i,t}))$ is enough. We see that instead of communicating the whole local model $w_{i,t}$, they only need to communicate the \emph{local residual} $f(x_{i,t}^l; w_{i,t})$.
Unfortunately, we are unable to analyze this algorithm. If fact, the update rules \eqref{eqn: fp client} and \eqref{eqn: fp server} are related to the \emph{fictitious play} strategy in two-player cooperative games, where each learner plays the best response to the other agent's empirical behavior in the past. In general, fictitious play takes the learner an exponentially long time to converge \cite{monderer1996fictitious, brandt2010rate}. In Appendix~\ref{sec: comparing three}, we give an example showing that if the models are badly initialized, the convergence of the update rules \eqref{eqn: fp client} and \eqref{eqn: fp server} can indeed be very slow, compared to \eqref{eqn: client update rule erm} and \eqref{eqn: server update rule erm}.

Fortunately, in the next subsection, we have a communication and computational efficient algorithm that avoids all the above issues.

\subsection{SGD-based approach}
\label{subsec: sgd}

SGD is a commonly used stochastic optimization method for differentiable losses. To apply SGD to the federated setting, a natural idea is that upon receiving a new sample, the clients and the server perform individual updates using the gradient with respect to local and global parameters, respectively. We begin with two natural baseline update rules that implement this intuition, and highlight the issues with them before describing our update rule which gets around these issues.

\subsubsection{Challenges with some baselines}
Perhaps the most natural update rule for performing SGD on both client and server sides, in the presence of client-dependent delays is the following:
    \begin{align*}
        w_{i,t+1} &= w_{i,t} - \eta_i\nabla^l \ell_{i,t}(w_{t-\down_i}^g, w_{i,t}) \\
        w_{t+1}^g &= w_{t}^g - \eta \sum_{i=1}^P \nabla^g \ell_{i,t-\up_i} (w^g_t, w_{i,t-\up_i})
    \end{align*}
This update is a direct adaptation of the ERM algorithm. However, we are unable to show a similar regret bound for it as in Theorem~\ref{theorem: main for erm}, where the delay dependence is in a lower order term of the regret.


The problem of this update rule is that the updates of the clients and the server are \emph{mis-aligned}. Observe that the prediction model pair is $(w^g_{t-\down_i}, w_{i,t})$ on the client side, with the global model lagging behind the local model by an amount of $\down_i$. However, the server is performing gradient descent on the model pair $(w^g_t, w_{i, t-\up_i})$, where the local model is behind the global model. This slight mismatch makes the global parameter update to a slightly incorrect direction.

A natural remedy to this mis-alignment is to instead perform the following updates:
\begin{align*}
        w_{i,t+1} &= w_{i,t} - \eta_i\nabla^l \ell_{i,t}(w_{t-\down_i}^g, w_{i,t}) \\
        w_{t+1}^g &= w_{t}^g - \eta \sum_{i=1}^P \nabla^g \ell_{i,t-\up_i} (w^g_{t-\up_i-\down_i}, w_{i,t-\up_i})
    \end{align*}

That is, the updates always utilize a gradient evaluated at a pair of models $(w^g_{t-\down_i}, w_{i,t})$ for some client $i$ and time $t$. While this update rule has the right pairing of local and global models on both client and server, there is an asymmetry in the delays experienced by the two. For the clients, there is effectively no delay in that the local model always updates from the most current local model. On the other hand, the server experiences a round-trip delay of $\up_i + \down_i$ in order to maintain alignment with the most current local model it has access to for client $i$. This asymmetry presents some technical challenges in our analysis, and results in a delay dependence on the dominant term in the regret. We note that unlike the mis-alignment issue, it is plausible that this challenge can be handled by a more careful analysis. However, we now present a different solution by creating a symmetric delayed setting on both client and server ends.

\subsubsection{Our algorithm and results}

To address the aforementioned problems, we \emph{align} the model updates as well as the delay structures on both client and server. That is, all gradients are taken on model pairs of the form $(w_{t-\down_i}^g, w_{i,t})$ and the client also experiences a similar delay as the server. To achieve the latter, we let the client make \emph{delayed updates}: in~\eqref{eqn: local model update}, the client performs a descent step using a gradient that is one round-trip delayed. The final algorithms are shown in Algorithm~\ref{alg: SGD client} and \ref{alg: SGD server} for the clients and the server respectively. With this fix, we can now obtain a similar result to the ERM case --- \emph{the delay only appears in a lower-order term of the regret}:

\begin{algorithm}[t]
    \caption{\alg.SGD.Client}
    \label{alg: SGD client}
    \For{$t=1,\ldots, T$}{
        Fetch the global model $w_{t-\down_i}^g$. \\
        Update local model:
        \begin{align}
            w_{i,t}\leftarrow \Pi_D\Big\{w_{i,t-1} - \eta_i   \nabla^l_{i,t-\down_i-\up_i}\Big\},    \label{eqn: local model update}
        \end{align}
        where $\nabla^l_{i,s} \triangleq \nabla^l \ell_{i,s}(w^g_{s-\down_i}, w_{i,s})$. \\
        \ \\
        Use the model pair $(w_{t-\down_i}, w_{i,t})$ to make predictions. \\
        Observe a new sample $z_{i,t}=(x_{i,t}^g, x_{i,t}^l, y_{i,t})$. \\
        Send $Z_{i,t}=(x_{i,t}^g,  f(x_{i,t}^l; w_{i,t}), y_{i,t})$ to the server.
    }
\end{algorithm}

\begin{algorithm}[t]
    \caption{\alg.SGD.Server}
    \label{alg: SGD server}
    \For{$t=1,\ldots, T$}{
        Receive $Z_{i,t-\up_i}$ from all $i=1,\ldots, P$.
        \ \\
        Update global model:
        \begin{align}
            w^g_{t}\leftarrow \Pi_D\Bigg\{w^g_{t-1} - \eta\sum_{i=1}^P \nabla^g_{i,t-\up_i}\Bigg\}    \label{eqn: global model update}
        \end{align}
        where
        \begin{align*}
            \nabla^g_{i,s}
            &\triangleq \nabla^g \ell_{i,s}(w^g_{s-\down_i}, w_{i,s})\\
            &= \nabla_{w} \ell\left(y_{i,s}, f(x_{i,s}^g; w), f(x_{i,s}^l; w_{i,s}) \right)\Big|_{w=w^g_{s-\down_i}} \tag{computable from $Z_{i,s}$}
        \end{align*}

    }
\end{algorithm}


\begin{theorem}
    \label{theorem: main for sgd}
    Suppose the variance of the gradient of the losses $\Var[\nabla \ell_{i,t}(w^g, w_i)]$ is upper bounded by $\sigma^2$ for all $i,t$. Then \alg.SGD (Algorithm~\ref{alg: SGD client} and \ref{alg: SGD server}) guarantees that
    \begin{align}
        &\E\left[ \frac{1}{PT}\sum_{i=1}^P \sum_{t=1}^T \ell_{i,t}\left(w^g_{t-\down_i}, w_{i,t}\right) - \ell_{i,t}\Big(w^g_*, w_{i,*}\Big)\right] \nonumber \\
        &=\order\left(\sqrt{\frac{\left(\|w^g_*\|^2 + \sum_{i=1}^P \|w_{i,*}\|^2 \right) \sigma^2}{PT}} \right) \nonumber  \\
        &\qquad \quad + \order\left(\frac{\left(\gamma D^4 G^2 \tau^2\right)^{\frac{1}{3}}}{T^{\frac{2}{3}}} + \frac{DG\tau}{T}\right).  \label{eqn: regret bound sgd}
    \end{align}
\end{theorem}


The complete proof of this theorem is provided in Appendix~\ref{sec:proof_sgd}. The techniques used in the analysis are inspired by those used in \cite{agarwal2011distributed}, which considers SGD in a delayed-feedback scenario and makes the dependence on delay only appeared in a lower-order term.
Similar to Theorem~\ref{theorem: main for erm}, we see that except for the additional regret caused by delay, the bound in Theorem~\ref{theorem: main for sgd} is an improvement over
\begin{align*}
        \order\left(\sqrt{\frac{\left(\sum_{i=1}^P \left(\|w^g_*\|^2+\|w_{i,*}\|^2\right)\right)\sigma^2}{PT}}\right),
\end{align*}
which is the achievable bound when all clients run independent SGD and compare their performance with the same benchmark $(w^g_*, w_{i,*})$.

\section{Reducing the communication through mini-batches} 
\label{section: mini-batch}
Our algorithms have heavy communication since the clients fetch a new global model each round.  This communication cost can be reduced by using mini-batches where both the clients and the server update their models once per batch. This can thus largely reduce the downlink communication because the client only needs to fetch the global model once per batch. The analysis in this section is inspired by the work of \citet{dekel2012optimal}. 

To analyze the algorithm with mini-batches, we can reuse our theorems developed in the previous sections. For example, in the \alg.SGD algorithm, if we use mini-batches of size $b$, we can define the aggregated loss
\begin{align}
    \widehat{\ell}_{i,n}(w^g, w_i) = \frac{1}{b}\sum_{t=(n-1)b+1}^{nb}\ell_{i,t}(w^g, w_i), \label{eqn: aggregated loss seq}
\end{align}
and run \alg.SGD for rounds $n=1,\ldots, \frac{T}{b}$. In the original algorithm, the clients accesses the global model $T$ times, but in the mini-batched algorithm, the clients only accesses $\frac{T}{b}$ times. We can also reuse Theorem~\ref{theorem: main for sgd} to analyze the regret of the batched algorithm. Applying Theorem~\ref{theorem: main for sgd} to the aggregated loss sequence defined in \eqref{eqn: aggregated loss seq}, we get
\begin{align*}
    &\E\left[ \frac{b}{PT} \sum_{i=1}^P \sum_{n=1}^{T/b}\widehat{\ell}_{i,n}(w_{n-\down_i'}^g, w_{i,n}) - \widehat{\ell}_{i,n}(w_g^*, w_{i,*}) \right] \\
    &=\order\left(\sqrt{\frac{\left(\|w_*^g\|^2 + \sum_{i=1}^P\|w_{i,*}\|^2 \right)\sigma^{'2}b}{PT}} \right)\\
    &\qquad + \order\left( \frac{(\gamma D^4 G^2\tau^{'2})^{\frac{1}{3}}b^{\frac{2}{3}}}{T^{\frac{2}{3}}}  + \frac{DG\tau' b}{T}\right) 
\end{align*}
where $\tau'=\frac{\tau}{b}+1$ is the delay counted in batches and $\sigma^{'2} = \frac{\sigma^2}{b}$ is the variance of the $\widehat{\ell}_{i,t}(w^g, w_i)$. The left-hand side turns out to be the true average loss of the learner, and the right-hand side is 
\begin{align*}
    &\order\left(\sqrt{\frac{\left(\|w_*^g\|^2 + \sum_{i=1}^P \|w_{i,*}\|^2  \right)\sigma^{2}}{PT}} \right)\\
    &\qquad  + \order\left(\frac{\left(\gamma D^4 G^2 (b+\tau)^2\right)^{\frac{1}{3}}}{T^{\frac{2}{3}}} + \frac{DG(b+\tau)}{T}\right).
\end{align*}
As one can see, the dominant term remains the same order, and the lower-order term is unaffected if $b<\tau$.

\section{Application: Contextual Bandits}
\label{sec: cb}
In this section, we demonstrate a specific application of our federated residual learning algorithms in the contextual bandit (henceforth CB) setting, a framework that is suitable to model recommendation systems and a variety of other online decision making settings.\footnote{See e.g. the ICML tutorial \url{https://hunch.net/~rwil/} and references therein for an overview} We show that our federated learning framework can be directly combined with the regression-based approach for CBs \citep{agarwal2012contextual, foster2018practical}. This enables CB learning to leverage advantage of personalization to individual clients while leveraging joint learning across multiple users as in a fully centralized setting, while prior approaches typically rely only on centralized learning~\citep{agarwal2016making}.

The protocol of the traditional (i.e., with single client) CB problem is as follows: at each round $t$,
\begin{itemize}
    \item Learner receives contexts $x_{t}(a)\in \mathbb{R}^d$ for all actions $a\in [K]$.
    \item Learner predicts an action $a_t\in [K]$.
    \item Learner observes the reward of the chosen action $r_t(a_t)\in[0,1]$.
\end{itemize}
In the regression-based CB setting, the learner has access to a class of regressors, which consists of functions from $\mathbb{R}^d$ to $[0,1]$. We suppose that the regressors are parametrized by $w$, and regressors can be written as $f(\cdot~; w)$. By the \emph{realizability assumption}, there is a regressor parametrized by $w^*$ that realizes the reward:
\begin{align*}
     \E[r_{t}(a)~|~ x_{t}(a)] = f(x_{t}(a); w^*).
\end{align*}
To evaluate the performance of the learner, we define the regret of the learner as
\begin{align*}
\Reg_{\cb} &= \E\left[\frac{1}{T}\sum_{t=1}^T r_t(a_t^*) -  r_t(a_t) \right]  \\
&= \E\left[\sum_{t=1}^T \max_{a\in[K]} f(x_{t}(a); w^*) -  f(x_{t}(a_t); w^*) \right],
\end{align*}
where $a_t^*=\argmax_{a} f(x_{t}(a); w^*)$ is the action chosen by the best regressor.

\paragraph{Federated CB setting.} In the federated CB setting, we assume that the reward for client $i$ can be joint realized by a global model $w^g_*$ and a local model $w_{i,*}$:
\begin{align*}
    \E[r_{i,t}(a)~|~x_{i,t}(a)] = f(x_{i,t}(a); w^g_*, w_{i,*}) \triangleq f_i^\star(x_{i,t}(a)).
\end{align*}
For example, in the residual learning scenario that we focus on in the previous sections, $f(x_{i,t}(a); w^g_*, w_{i,*}) = f(x_{i,t}^g(a); w^g_*) + f(x_{i,t}^l(a); w_{i,*})$,
where $x_{i,t}^g(a)$ and $x_{i,t}^l(a)$ are the global and local contexts (features) of client $i$ that correspond to action $a$ at time $t$. Let $a_{i,t}$ be the action chosen by client $i$ at time $t$. The regret is defined as
\begin{align}
\label{eqn: CB regret}
    &\Reg_{\cb}
    = \E\left[ \frac{1}{PT}\sum_{t=1}^T \sum_{i=1}^P r_{i,t}(a_{i,t}^*) - r_{i,t}(a_{i,t}) \right]\\\nonumber
    &= \E\left[\frac{1}{PT}\sum_{t=1}^T \sum_{i=1}^P \max_{a} f_i^\star(x_{i,t}(a)) - f_i^\star(x_{i,t}(a_{i,t}))  \right].
\end{align}

\subsection{$\epsilon$-greedy with federated regression}
\label{subsec: epsilon-greedy for CB}
Bandit problems are difficult than usual supervised learning problems due to the limited feedback (i.e., the learner only observes the reward of the action she picks in that round). To deal with this lack of information, in every round the $\epsilon$-greedy strategy uses a small probability $\epsilon\in (0,1)$ to randomly pick an action. When the data is i.i.d. across time, an alternative implementation is to perform exploration every $B=\frac{1}{\epsilon}$ rounds (we will use this version to simplify the presentation). The learner uses the data collected from these exploration rounds to update the model parameters $w_t^g, w_{i,t}$; for other rounds, the learner simply chooses actions based on the current parameters. More precisely, on each round of $t=B, 2B, 3B, \ldots$, each client uniformly randomly picks an action from $[K]$ (i.e., $a_{i,t}\sim \text{Uniform}\{[K]\}$), and feeds the following loss to \alg:
\begin{align*}
\ell_{i,t}(w^g, w_i)=\bigg( r_{i,t}(a_{i,t}) - f\left(x_{i,t}(a_{i,t}); w^g, w_{i}\right) \bigg)^2.
\end{align*}
In other rounds, all clients simply choose the following action and do not update the models:
\begin{align}
a_{i,t} = \argmax_{a\in[K]} f\left(x_{i,t}(a); \widehat{w}^g_t, \widehat{w}_{i,t}\right),  \label{eqn: CB choose action}
\end{align}
where $\widehat{w}^g_t, \widehat{w}_{i,t}$ are the global and local models maintained by client $i$ at time $t$ respectively.  The above algorithm has the regret guarantee given by the following theorem.

\begin{theorem}
     \label{theorem: cb}
     With the above algorithm for federated contextual bandits, the regret can be upper bounded as follows:
     \begin{align*}
          &\E\left[ \frac{1}{PT}\sum_{t=1}^T \sum_{i=1}^P r_{i,t}(a_{i,t}^*) - r_{i,t}(a_{i,t}) \right] \\
          &=  \mathcal{O}\Bigg( \left(\frac{K^4\left(\|w_*^g\|^2 + \sum_{i=1}^P \|w_{i,*}\|^2\right)\sigma^2}{PT}\right)^{\frac{1}{5}}  \\
          &\qquad +  \frac{\text{poly}\left(K, D,G,\gamma,\tau \right)}{T^{\frac{1}{4}}} \Bigg).
     \end{align*}
     if $B$ is chosen optimally (see Appendix~\ref{section: proof of Theorem CB} for the precise expression of the lower-order term).
\end{theorem}
The proof of Theorem~\ref{theorem: cb} is give in Appendix~\ref{section: proof of Theorem CB}. We note that the regret of this approach is sub-optimal in its dependence on $T$ when we compare with the best achievable rates in a fully centralized setting, owing to the use of a remarkably simple CB algorithm here for a proof of concept. However, even in this simple case, we observe as before that the delay only affects asymptotically non-dominant terms and furthermore does not influence the choice of the exploration level for the algorithm (as captured in the setting of $B$ in Appendix~\ref{section: proof of Theorem CB} which does not depend on $\tau$). In future work, it would be interesting to study how the optimal UCB-like approaches can be adapted to work in the federated setting through the similar use of federated regression oracles like we have done here for better dependence on $T$.

\section{Experiments}
\label{sec: exp}
To test our algorithms, we create datasets that mimic the federated learning scenario.

\subsection{Dataset generation}
\label{subsec: data generation}
We create binary classification from real multiclass classification datasets provided in LIBSVM Dataset \cite{chang2011libsvm} as follows:
\begin{itemize}
    \item For a multiclass classification dataset with the set of classes being $[K]=\{1,2,\ldots, K\}$, we randomly pick a subset $\calA$ of it. All data samples from $\calA$ are merged as a new class $C_0$.
    \item For each client, its assigned task is a binary classification problem between class $C_0$ and a random class from $[K]\backslash \calA$.
\end{itemize}
As can be seen, different clients face different classification problems which might be related: Suppose Client 1's task is to distinguish $C_0$ from class $A$; Client 2's task is to distinguish $C_0$ from class $B$. When there exists a single hyperplane that saperates $C_0$ from $A$ and $B$ well, then the two clients' task are closely related, although this is not guaranteed in the datasets we generate.

We then assign data to workers so that the following two properties are satisfied:
\begin{enumerate}
    \item Different clients may work on the same task (i.e., the same random class from $[K]\backslash \calA$), but the examples they are assigned to are guaranteed to be disjoint.
    \item The positive and negative examples assigned to each client are roughly balanced.
\end{enumerate}

In order to let the property 1 above hold, each client is assigned at most $\frac{\text{\# examples belonging to $\calA$}}{\text{\# clients}}$ data samples. In order to make this large enough for experimental purpose, $|\calA|$ should not be too small; on the other hand, in order to keep the task diversity of the clients, $K-|\calA|$ should also not be too small. We simply make a balanced choice of $|\calA|=\lfloor 0.3K \rfloor$.

In order to satisfy the two properties, we distribute the data to clients following the procedures below:
\begin{enumerate}
    \item Uniformly randomly distribute the samples of $\calA$ to all clients. Suppose each client receives $N$ samples in this stage. We set an upper bound $N_0$ so that $N\leq N_0$.
    \item Create \emph{buckets} of data samples from $[K]\backslash \calA$. Each bucket contains $N$ single-class samples.
    \item Each client is randomly assigned a bucket.
\end{enumerate}
At the end, each client has $2N$ samples with balanced classes.

In order to maintain the diversity of tasks, we pick from LIBSVM multiclass classification datasets that have no less than $6$ classes.

For the original feature vector of dimension $d$, we randomly make $\frac{d}{2}$ of them the global features and the other $\frac{d}{2}$ the local features.

\subsection{Test algorithms and implementation}
\label{subsec: test algorithms and implementation}
We test and compare three algorithms under the SGD framework:
\begin{enumerate}
    \item \textbf{Independent}: Each client performs individual SGD on their own dataset using the full set of features (i.e., global features plus local features).
    \item \textbf{Central}: The server runs SGD over the aggregated dataset from all clients using global features.
    \item \textbf{\alg}: \alg.SGD with the server learning on global features and the clients learning on local features.
\end{enumerate}
The first two algorithms are our baselines that correspond to fully-local and full-central solutions. We do not make the server learn on local features because in general local features can be differently defined by each client (and not all clients may want to share local features).

We use the linear regression implementation by Vowpal Wabbit (VW) \cite{VW} The VW command
we use for the linear regression model is ``{-}{-}adaptive''.

In all experiments we describe below, we set $N_0$ defined above to be $30$, meaning that each client has at most $60$ data samples. This simulates a regime where each client has relatively few data samples. For each experiment, we run the algorithms for $T=500$ rounds (so a training dataset may train for multiple epochs), and then test the performance on a held-out test dataset. Each number in the figures is an average over $50$ random rollouts.

\subsection{General comparison with the baselines}
We first make a general comparison among three methods.  We test under $P=10$ (in Figure~\ref{fig:10 worker overall}) and $P=50$ (Figure~\ref{fig:50 worker overall}). From the figures, we see that the \alg approach is always a near winner and sometimes greatly outperforms the baselines.

\begin{figure}[t]
    \centering
    \includegraphics[width=0.5\textwidth,trim={0 0pt 0 0},clip]{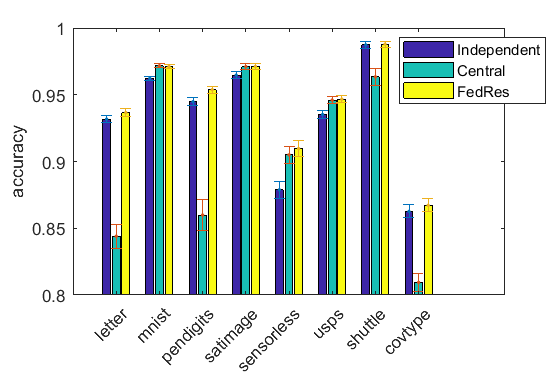}
    \caption{Accuracy of Independent, Central, and \alg approaches with $P=10$ clients in the absence of communication delays.}
    \label{fig:10 worker overall}
\end{figure}

\subsection{Robustness to task similarity}
\label{subsec: experiment for robustness}
We can observe from Figure~\ref{fig:50 worker overall} and~\ref{fig:10 worker overall} that there are two types of datasets: those for which Independent outperforms Central (letter, pendigits, shuttle, covtype), and those for which Central outperforms Independent (mnist, satimage, sensorless, usps). Intuitively, Central should outperform Independent when the tasks for different clients are similar, and on the contrary, Independent should outperform Central when the tasks are different in general (so aggregating the data hurts the performance). The former is the case when federated learning has benefits over independent client-side training. One can foresee that in this case, when the number of clients increases, the overall performance should improve because each client benefits from the effectively increased number of data samples. We indeed observe this phenomenon in Figure~\ref{fig: the performance with number of clients group 1}, where we plot the performance on the sensorless and mnist dataset. On the other hand, for datasets like letter and pendigits, where Independent performs better than Central, the performance of federated learning should improve little with the number of clients. This can also be observed from Figure~\ref{fig: the performance with number of clients group 2}, where we plot for letter and pendigits.

In all the experiments, \alg is always comparable with the best of Independent and Central, we can conclude that \alg is \emph{robust} to task similarity. That is, when the data distributions are similar across clients, the global model in our algorithm will take effect and bring the benefits of joint training; when the tasks are not similar, in which case using the global model might be harmful, our local model still keeps the performance of independent training.

\begin{figure}[t]
\subfigure[sensorless]{
   \includegraphics[width=8cm]{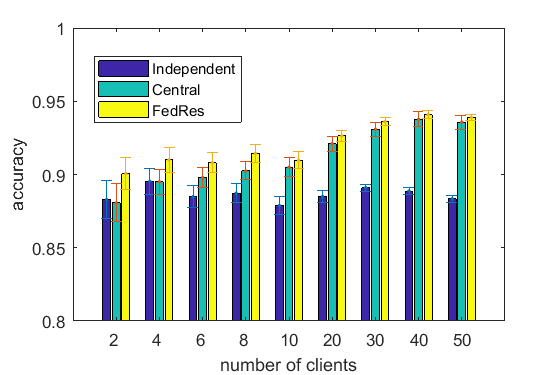}
 }\hfill
\subfigure[mnist]{
   \includegraphics[width=8cm]{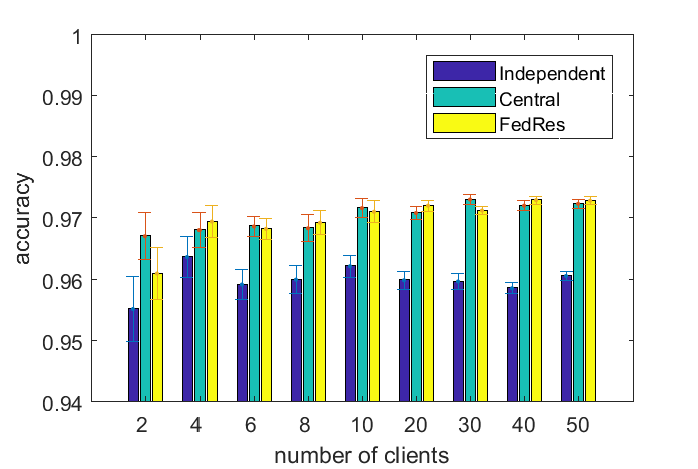}
 }\hfill

\caption{Test accuracy versus the numbers of clients for sensorless and mnist datasets. In these experiments, we let the delay be zero. Each data point is an average over $50$ random trials. }
 \label{fig: the performance with number of clients group 1}
\end{figure}

\begin{figure}[t]
\subfigure[letter]{
   \includegraphics[width=8cm]{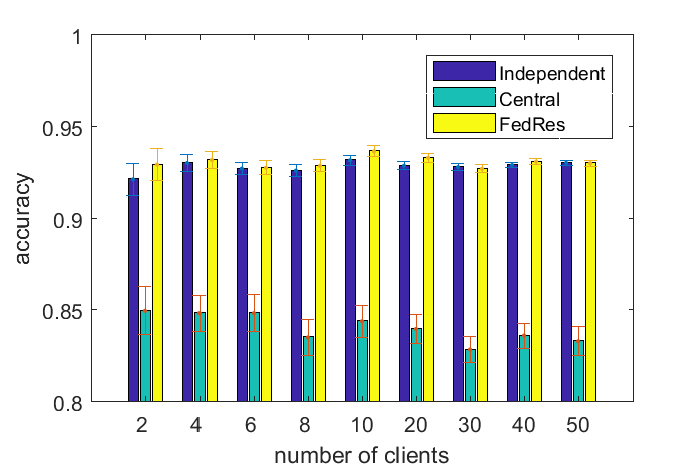}
 }\hfill
\subfigure[pendigits]{
   \includegraphics[width=8cm]{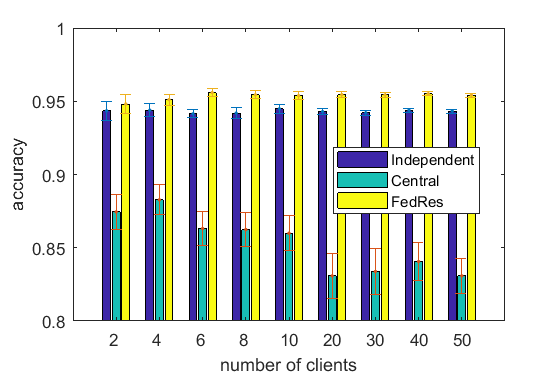}
 }\hfill

\caption{Test accuracy versus the numbers of clients for letter and pendigits datasets. In these experiments, we let the delay be zero. Each data point is an average over $50$ random trials. }
 \label{fig: the performance with number of clients group 2}
\end{figure}

\subsection{Effect of delay}
\label{subsec: effect of delay}
We also empirically test the effect of delay on the performance of the system in Figure~\ref{fig: plot_of_delay_main} which shows a modest degradation in performance with delay in two of our datasets. For more experimental results on the effect of delay and comparison with baseline algorithms, please see Appendix~\ref{sec:exp-details}.
\begin{figure}[t]
\subfigure[letter]{
   \includegraphics[width=8cm]{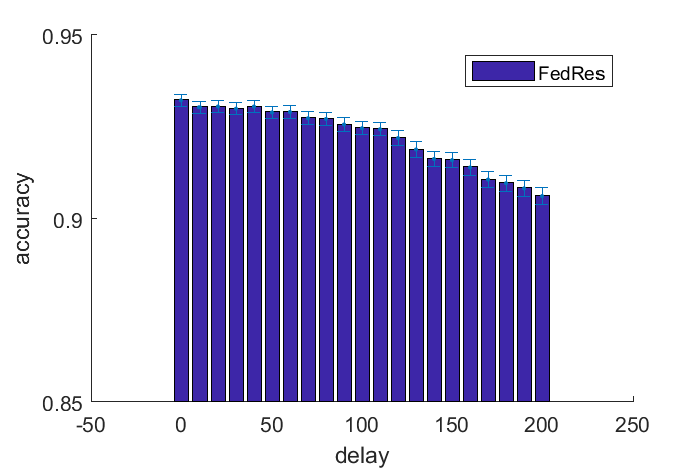}
 }\hfill
\subfigure[pendigits]{
   \includegraphics[width=8cm]{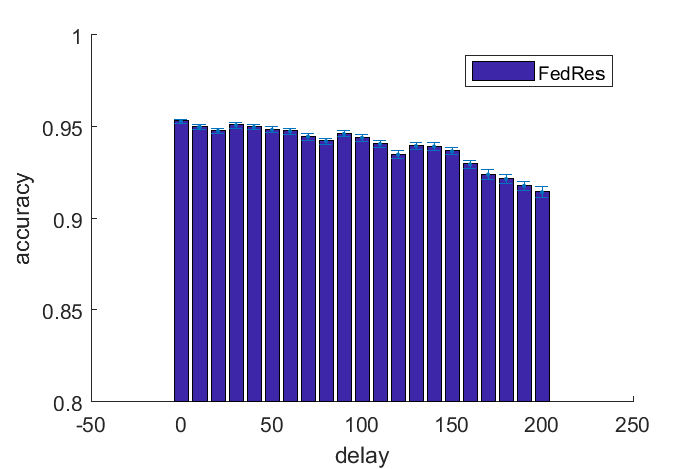}
 }\hfill

    \caption{Training on \textit{letter} and \textit{pendigits} with varying delays. Test accuracy versus delay for the letter dataset. We let the number of clients be $50$. Each data point is an average over $50$ random trials.}
    \label{fig: plot_of_delay_main}
\end{figure}
\section{Conclusion}
We proposed a new framework of federated learning in which simple extensions of ERM and SGD-style algorithms enable personalization in an efficient manner, both theoretically and empirically. While personalization was the primary goal here, only sharing local predictions to the server has useful consequences for privacy as well.

\bibliography{fedres}
\bibliographystyle{icml2020}

\appendix
\onecolumn
{\LARGE \textbf{Appendix}}\\
\ \\
We include the following items in the appendix:
\begin{center}
\begin{itemize}
     \item[A.] The proof of Theorem~\ref{theorem: main for erm} for the \alg.ERM algorithm
     \item[B.] The proof of Theorem~\ref{theorem: main for sgd} for the \alg.SGD algorithm
     \item[C.] Explaining the failure of the fictitious-play strategy described in Eq.~\ref{eqn: fp client} and~\ref{eqn: fp server} with simulation results
     \item[D.] The proof of Theorem~\ref{theorem: cb} for federated contextual bandits
     \item[E.] More experimental results that complement Section~\ref{sec: exp}
\end{itemize}
\end{center}
Specifically, in Section~\ref{subsec: effect of number of clients}, we provide the results of ``accuracy versus number of clients'' for the omitted datasets in Section~\ref{subsec: experiment for robustness}. In Section~\ref{subsec: effect of delay in appendix}, we conduct more extensive experiments on the effect of delays, and compare different schemes, making Section~\ref{subsec: effect of delay} more complete. In Section~\ref{subsec: combined plots}, we provide more ``accuracy versus number of clients'' plots under different amounts of delay. In Section~\ref{app exp conclusion}, we give a short conclusion for what we observe from the experiments.

\section{Proofs for Theorem~\ref{theorem: main for erm} (\alg.ERM algorithm)}
We define several notations to be used in the proofs.
\begin{definition}
    For any $w^g, w^l$,
    \begin{align*}
        L_i(w^g, w^l) &\triangleq \E\left[ \ell_{i,t}(w^g, w^l) \right] \\
        (w^g_*, w_{1,*}, \ldots, w_{P,*}) &\triangleq \argmin_{w^g, w_1,\ldots, w_P} \sum_{i=1}^P L_i(w^g, w_i) \\
        \hatL_{i,t}(w^g, w^l) &\triangleq \frac{1}{t-1}\sum_{s=1}^{t-1} \ell_{i,s}(w^g, w^l) \\
        \Delta_{i,t}(w^g, w^l) &\triangleq \hatL_{i,t}(w^g, w^l) - \hatL_{i,t}(w^g_*, w_{i,*}).
    \end{align*}
\end{definition}

\begin{definition}
    \label{definition: d sigma}
    Define
    \begin{align*}
        \overline{\sigma} &= \sqrt{\frac{1}{P}\sum_{i=1}^P \sigma_i^2}, \\
        \overline{d} &= \frac{1}{P}\left(d+\sum_{i=1}^P d_i\right),
    \end{align*}
    where $\sigma_i$ is an upper bound for the variance of $\ell_{i,t}(w^g, w_i)$ for any $w^g, w_i$, and $d, d_1, \ldots, d_P$ are the dimensions of $w^g, w_1, \ldots, w_P$ respectively.
\end{definition}

First, we bound the difference between $\sum_{i=1}^P \widehat{L}_{i,t}$ and $\sum_{i=1}^P L_{i}$.
\begin{lemma}
\label{lemma: Lhat and L}
Suppose $D\leq T$. With probability $1-\frac{1}{T}$, the following holds for all $t$ and all $(w^g, w_1, \ldots, w_P)$:
\begin{align*}
    \Bigg\vert \sum_{i=1}^P \widehat{L}_{i,t}(w^g, w_i) - \sum_{i=1}^P L_i(w^g, w_i) \Bigg\vert = \order\left( P\cdot\sqrt{\frac{\overline{\sigma}^2 \overline{d}\log T}{t-1}} + P^2\cdot  \frac{\overline{d}\log T}{t-1}  \right).
\end{align*}
\end{lemma}
\begin{proof}
    We use Bernstein's inequality on the discretized space of $(w^g, w_1, \ldots, w_P)$. Recall that $(w^g, w_1, \ldots, w_P)\in \mathbb{R}^d\times \mathbb{R}^{d_1} \times \cdots \mathbb{R}^{d_P}$. We discretize each dimension into $T^2$ values, and so the total number of discretization points is $\left(T^2\right)^{d+\sum_{i=1}^P d_i}$. Suppose the nearest discretization point to $(w^g, w_1, \ldots, w_P)$ is $(\widehat{w}^g, \widehat{w}_1, \ldots, \widehat{w}_P)$. By Bernstein's inequality, with probability at least $1-\frac{1}{T^2}$ the following holds for all discretization points:
    \begin{align}
        &\left\lvert \sum_{i=1}^P \widehat{L}_{i,t}(\widehat{w}^g, \widehat{w}_i) -\sum_{i=1}^P L_i(\widehat{w}^g, \widehat{w}_i)\right\rvert \nonumber \\
        &= \left\lvert \sum_{i=1}^P \frac{1}{t-1}\sum_{s=1}^{t-1} \ell_{i,t}(\widehat{w}^g, \widehat{w}_i) -\sum_{i=1}^P L_i(\widehat{w}^g, \widehat{w}_i)\right\rvert \nonumber \\
        &= \order\left( \sqrt{\frac{\left(\sum_{i=1}^P \sigma_i^2\right)\left(d+\sum_{i=1}^P d_i\right)\log T}{t-1}} + \frac{P\left(d+\sum_{i=1}^P d_i\right)\log T}{t-1}  \right) \tag{$\ell_{i,t}\in[0,1]$} \\
        &= \order\left( P\cdot\sqrt{\frac{\overline{\sigma}^2 \overline{d}\log T}{t-1}} + P^2\cdot \frac{\overline{d}\log T}{t-1}  \right). \label{eqn: concentrated1}
     \end{align}
     The first equality comes from the fact that all clients generate data independently, so the variance of $\sum_{i=1}^P \ell_{i,s}(\widehat{w}^g, \widehat{w}_{i})$ is upper bounded by $\sum_{i=1}^P \sigma_i^2$. The $(d+\sum_{i=1}^Pd_i)\log T$ factor comes from $\log\left((T^2)^{d+\sum_{i=1}^P d_i} \right)$.
    Since the distance between $(\widehat{w}^g, \widehat{w}_1, \ldots, \widehat{w}_P)$ and $(w^g, w_1, \ldots, w_P)$ is no more than $\frac{D}{T^2}$ in each dimension, the above implies that
    \begin{align}
        &\left\lvert \sum_{i=1}^P \frac{1}{t-1}\sum_{s=1}^{t-1}\ell_{i,s}(w^g, w_i) -\sum_{i=1}^P L_i(w^g,w_i)\right\rvert \nonumber
        =\order\left( P\cdot\sqrt{\frac{\overline{\sigma}^2 \overline{d}\log T}{t-1}} + P^2\cdot \frac{\overline{d}\log T}{t-1} + \frac{PD\overline{d}}{T^2} \right)
    \end{align}
    holds with probability $1-\frac{1}{T^2}$ for all $w^g, w_i$. Using a union bound over $t$ finishes the proof.
\end{proof}

Next, we state a lemma that is useful for showing the convergence of alternating minimization, which is adapted from the analysis in \cite{beck2015convergence}.

\begin{lemma}
    \label{lemma: alternating key lemma}
    Let $\ell(u,v)$ be a $\gamma$-smooth joint convex function of $u$ and $v$, and $\Omega_u$, $\Omega_v$ are convex feasible sets of $u$, $v$ respectively. Now fix $u=u_0$, and let $v_0 = \argmin_{v\in\Omega_v}\ell(u_0, v)$. Suppose $\sup_{u\in\Omega_u}\|u\|\leq D$ and $\ell(u,v)\in[0,R]$ for any $u,v$. Then
    \begin{align*}
        \min_{u\in\Omega_u}\ell(u,v_0)\leq \ell(u_0, v_0) - \frac{1}{18\gamma D^2 + 2R}\left[\ell(u_0, v_0) - \ell(u_*,v_*)\right]_+^2.
    \end{align*}
    for any $u_*\in\Omega_u, v_*\in\Omega_v$.
\end{lemma}
\begin{proof}
    Define
    \begin{align*}
        u_1 = \argmin_{u \in \Omega_u }  \Big\| u - u_0 + \frac{1}{\gamma} \nabla_u \ell(u_0,v_0) \Big\|^2.
    \end{align*}


    By the smoothness of $\ell$, we have
    \begin{align}
        \ell(u_1, v_0)\leq \ell(u_0, v_0) + \nabla_u\ell(u_0, v_0)^\top (u_1-u_0) + \frac{\gamma}{2}\|u_1-u_0\|^2.    \label{eqn: alternating min tmp}
    \end{align}
    By the optimality of $u_1$, we have
    \begin{align}
        \left(u_1 - u_0 + \frac{1}{\gamma}\nabla_u\ell(u_0,v_0)\right)^\top \left(u'-u_1\right)\geq 0   \label{eqn: u eta optimality}
    \end{align}
    for all $u'\in\Omega_u$.

    Specially, by invoking \eqref{eqn: u eta optimality} with $u'=u_0$, we can further upper bound the right-hand side of \eqref{eqn: alternating min tmp} by
    \begin{align}
        \ell(u_0, v_0) -\gamma\|u_1-u_0\|^2 + \frac{\gamma}{2}\|u_1-u_0\|^2 \leq \ell(u_0, v_0) -\frac{\gamma}{2}\|u_1-u_0\|^2. \label{eqn: alternating min tmp 2}
    \end{align}
    Below we further lower bound $\|u_1-u_0\|^2$. Since $v_0$ is the minimizer of $\ell(u_0, \cdot)$ in $\Omega_v$, we have
    \begin{align}
        \nabla_v\ell(u_0, v_0)^\top (v'-v_0)\geq 0 \label{eqn: v 0 optimality}
    \end{align}
    for all $v'\in\Omega_v$.

    Define $(u_{\min}, v_{\min})=\argmin_{u\in\Omega_u, v\in\Omega_v}\ell(u,v)$.  With the above ingredients, we can bound
    \begin{align*}
        &\min_{u\in\Omega_u}\ell(u, v_0) - \ell(u_{\min}, v_{\min})\\
        &\leq \ell(u_1, v_0) - \ell(u_{\min}, v_{\min})\\
        &\leq \ell(u_0, v_0) - \ell(u_{\min}, v_{\min}) + \nabla_u\ell(u_0, v_0)^\top (u_1-u_0) + \frac{\gamma}{2}\|u_1-u_0\|^2.  \tag{by \eqref{eqn: alternating min tmp}} \\
        &\leq \nabla_u\ell(u_0, v_0)^\top (u_0-u_{\min}) + \nabla_v\ell(u_0, v_0)^\top (v_0-v_{\min}) + \nabla_u\ell(u_0, v_0)^\top (u_1-u_0) + \frac{\gamma}{2}\|u_1-u_0\|^2  \tag{by the convexity of $\ell$} \\
        &\leq \nabla_u\ell(u_0, v_0)^\top (u_1-u_{\min}) +  \frac{\gamma}{2}\|u_1-u_0\|^2 \tag{using \eqref{eqn: v 0 optimality} with $v'=v_{\min}$} \\
        &\leq \gamma(u_1-u_0)^\top (u_{\min}-u_1) + \frac{\gamma}{2}\|u_1-u_0\|^2  \tag{using \eqref{eqn: u eta optimality} with $u'=u_{\min}$}\\
        &\leq 2\gamma\|u_1-u_0\| D + \gamma\|u_1-u_0\|D \\
        &\leq 3\gamma\|u_1-u_0\|D,
    \end{align*}
    which implies
    \begin{align*}
        \|u_1-u_0\|^2 \geq \frac{1}{9\gamma^2D^2} \left[\min_{u\in\Omega_u}\ell(u, v_0) - \ell(u_{\min},v_{\min})\right]^2.
    \end{align*}
    Combine this with \eqref{eqn: alternating min tmp}, \eqref{eqn: alternating min tmp 2}, and using the fact $\min_{u\in\Omega_u}\ell(u,v_0)\leq \ell(u_1,v_0)$, we get
    \begin{align*}
        \min_{u\in\Omega_u}\ell(u,v_0)
        &\leq \ell(u_0, v_0) - \frac{1}{18\gamma D^2}\left[\min_{u\in\Omega_u}\ell(u, v_0) - \ell(u_{\min},v_{\min})\right]^2\\
        &= \ell(u_0, v_0) - \frac{1}{18\gamma D^2}\left[\ell(u_0, v_0) - \ell(u_{\min},v_{\min})\right]^2 \\
        &\qquad \qquad + \frac{1}{18\gamma D^2}\left(\ell(u_0,v_0)-\min_{u\in\Omega_u}(u,v_0)\right)\left(\ell(u_0,v_0)+\min_{u\in\Omega_u}(u,v_0) - 2\ell(u_{\min},v_{\min})\right) \\
        &\leq \ell(u_0, v_0) - \frac{1}{18\gamma D^2}\left[\ell(u_0, v_0) - \ell(u_{\min},v_{\min})\right]^2
        + \frac{R}{9\gamma D^2}\left(\ell(u_0,v_0)-\min_{u\in\Omega_u}\ell(u,v_0)\right).
    \end{align*}
    Rearranging this gives
    \begin{align*}
        \min_{u\in\Omega_u}\ell(u,v_0)
        &\leq \ell(u_0, v_0) - \frac{1}{18\gamma D^2 \left(1+\frac{R}{9\gamma D^2}\right)}\left[\ell(u_0, v_0) - \ell(u_{\min},v_{\min})\right]^2 \\
        &=\ell(u_0, v_0) - \frac{1}{18\gamma D^2 +2R}\left[\ell(u_0, v_0) - \ell(u_{\min},v_{\min})\right]^2 \\
        &\leq \ell(u_0, v_0) - \frac{1}{18\gamma D^2 +2R}\left[\ell(u_0, v_0) - \ell(u_*,v_*)\right]_+^2.
    \end{align*}
\end{proof}

\begin{lemma}
    \label{lemma: erm component 1}
    Let $\delta_{t,t'}$ denote $\frac{1}{P}\sum_{i=1}^P \Delta_{i,t'}(w^g_t, w_{i,t'})$. Then
    \alg.ERM (Algorithm~\ref{alg: ERM client} and~\ref{alg: ERM server}) ensures that for any $t>\tau+1$,
    \begin{align*}
        \delta_{t,t-\up}\leq \delta_{t-\down-\up, t-\up} - \frac{1}{18\gamma D^2+2}\left[\delta_{t-\down-\up, t-\up}\right]_+^2.
    \end{align*}
\end{lemma}
\begin{proof}
    By Algorithm~\ref{alg: ERM client}, $w_{i, t-\up}$ minimizes $\hatL_{i,t-\up}(w_{t-\down-\up}, \cdot)$, and therefore, $\left(w_{1,t-\up}, w_{2,t-\up}, \ldots, w_{P,t-\up}\right)$ jointly minimizes $\sum_{i=1}^P \hatL_{i, t-\up}(w^g_{t-\down-\up}, \cdot)$. Using Lemma~\ref{lemma: alternating key lemma} with $R=1$, we get
    \begin{align*}
        \frac{1}{P}\sum_{i=1}^P \hatL_{i,t-\up}(w^g_t, w_{i, t-\up}) \leq  \frac{1}{P}\sum_{i=1}^P \hatL_{i,t-\up}(w^g_{t-\down-\up}, w_{i, t-\up}) - \frac{1}{18\gamma D^2 + 2} \left[ \frac{1}{P}\sum_{i=1}^P  \Delta_{i,t-\up}(w_{t-\down-\up}^g, w_{i,t-\up})  \right]_+^2.
    \end{align*}
    Subtracting both sides with $\frac{1}{P}\sum_{i=1}^P \hatL_{i,t-\up}(w^g_*, w_{i, *})$ finishes the proof.
\end{proof}

\begin{lemma}
    \label{lemma: erm component 2}
     \alg.ERM (Algorithm~\ref{alg: ERM client} and~\ref{alg: ERM server}) ensures that for any $t>\tau+1$,
     \begin{align*}
         \delta_{t-\down, t}\leq \delta_{t-\down, t-\down-\up} + \order\left( \frac{\tau}{t}\sqrt{\frac{\overline{\sigma}^2 \overline{d}}{t-\tau}} + P\cdot \frac{\tau\overline{d}}{t(t-\tau)}  \right)
     \end{align*}


\end{lemma}

\begin{proof}
            \begin{align}
            &\sum_{i=1}^P \Delta_{i,t}(w^g_{t-\down}, w_{i,t}) \nonumber  \\
            &= \sum_{i=1}^P \left( \hatL_{i,t}(w_{t-\down}^g, w_{i,t}) - \hatL_{i,t}(w^g_{*}, w_{i,*})  \right) \nonumber \\
            &\leq \sum_{i=1}^P \left( \hatL_{i,t}(w_{t-\down}^g, w_{i,t-\down-\up}) - \hatL_{i,t}(w^g_{*}, w_{i,*}) \right) \tag{because $w_{i,t}$ is the minimizer of $\widehat{L}_{i,t}(w^g_{t-\down}, \cdot)$}  \\
            &= \sum_{i=1}^P \Delta_{t-\down-\up}(w_{t-\down}^g, w_{i,t-\down-\up}) +  \sum_{i=1}^P \left(\hatL_{i,t}(w_{t-\down}^g, w_{i,t-\down-\up}) - \hatL_{i,t-\down-\up}(w_{t-\down}^g, w_{i,t-\down-\up}) \right)    \nonumber \\
            &\qquad \qquad + \sum_{i=1}^P \left(\hatL_{i,t-\down-\up}(w^g_{*}, w_{i,*}) - \hatL_{i,t}(w^g_{*}, w_{i,*}) \right) . \label{eqn: tmpo}
        \end{align}
    Now remains the bound the last two terms above. Note that they are of similar form. Below, let $\tau=\down+\up$. Then for any $t>\tau+1$, any $(w^g, w_1, w_2, \ldots, w_P)$,
    \begin{align}
        &\sum_{i=1}^P \left( \hatL_{i,t}(w^g, w_i) - \hatL_{i, t-\tau}(w^g, w_i)  \right) \nonumber \\
        &= \sum_{i=1}^P \left(\frac{1}{t-1}\sum_{s=1}^{t-1}\ell_{i,s}(w^g, w_i) - \frac{1}{t-\tau-1}\sum_{s=1}^{t-\tau-1}\ell_{i,s}(w^g, w_i)\right)   \nonumber \\
        &= \sum_{i=1}^P \frac{1}{t-1}\left(\sum_{s=1}^{t-1}\ell_{i,s}(w^g, w_i) - \frac{t-1}{t-\tau-1}\sum_{s=1}^{t-\tau-1}\ell_{i,s}(w^g, w_i)\right)    \nonumber \\
        &= \sum_{i=1}^P \frac{1}{t-1}\left(\sum_{s=t-\tau}^{t-1}\ell_{i,s}(w^g, w_i) - \frac{\tau}{t-\tau-1}\sum_{s=1}^{t-\tau-1}\ell_{i,s}(w^g, w_i)\right)  \nonumber \\
        &=  \sum_{i=1}^P \frac{\tau}{t-1}\left(\frac{1}{\tau}\sum_{s=t-\tau}^{t-1}\ell_{i,s}(w^g, w_i) -  \widehat{L}_{i,t-\tau}(w^g, w_i)\right)   \label{eqn: t and t-tau difference}
    \end{align}

     For the second term on the right-hand side of \eqref{eqn: tmpo}, we can now bound its expectation with the help of \eqref{eqn: t and t-tau difference} and Lemma~\ref{lemma: Lhat and L}:
    \begin{align*}
       &\E\left[ \sum_{i=1}^P \left( \hatL_{i,t}(w_{t-\down}^g, w_{i,t-\down-\up}) - \hatL_{i,t-\down-\up}(w_{t-\down}^g, w_{i,t-\down-\up}) \right)  \right] \\
       &= \E\left[ \sum_{i=1}^P \frac{\tau}{t-1} \left(\frac{1}{\tau}\sum_{s=t-\tau}^{t-1}\ell_{i,s}(w^g_{t-\down}, w_{i,t-\down-\up}) - \widehat{L}_{i,t-\tau}(w^g_{t-\down}, w_{i,t-\down-\up})\right) \right] \\
       &\leq \underbrace{\E\left[ \sum_{i=1}^P \frac{\tau}{t-1} \left(\frac{1}{\tau}\sum_{s=t-\tau}^{t-1}\ell_{i,s}(w^g_{t-\down}, w_{i,t-\down-\up})
       - L_i(w^g_{t-\down}, w_{i,t-\down-\up})\right) \right]}_{\term_1} + \order\left( P\cdot\sqrt{\frac{\overline{\sigma}^2 \overline{d}\log T}{t-\tau-1}} + P^2\cdot \frac{\overline{d}\log T}{t-\tau-1}  \right) \times \frac{\tau}{t-1}
    \end{align*}
    Notice that $w_{t-\down}^g$ and $w_{i,t-\down-\up}=w_{i,t-\tau}$ only depend on $\ell_{i,s}$ for $s< t-\tau$. Therefore, conditioned on $\{\ell_{i,s}\}_{s< t-\tau}$, the expectation of $\ell_{i,s}(w^g_{t-\down}, w_{i,t-\down-\up})$ for $s\geq t-\tau$ is exactly $L_i(w^g_{t-\down}, w_{i,t-\down-\up})$. Therefore, $\term_1$ is zero. On the other hand, the expectation of the third term on the right-hand side of \eqref{eqn: tmpo} is
    \begin{align*}
       \E\left[\sum_{i=1}^P  \left(\hatL_{i,t-\down-\up}(w^g_{*}, w_{i,*}) - \hatL_{i,t}(w^g_{*}, w_{i,*}) \right) \right] = 0
    \end{align*}
    because $\hatL_{i,t}(w^g, w^l)$ is an unbiased estimator of $L_i(w^g, w^l)$ for fixed $(w^g, w^l)$. With all the above arguments, we can bound the expectation of the last two summations in \eqref{eqn: tmpo} by
    \begin{align*}
         \widetilde{\order}\left( P\cdot\sqrt{\frac{\overline{\sigma}^2 \overline{d}}{t-\tau-1}} + P^2\cdot \frac{\overline{d}}{t-\tau-1}  \right) \times \frac{\tau}{t-1},
    \end{align*}
    which finishes the proof.
\end{proof}

We also need the following lemma to prove Theorem~\ref{theorem: main for erm}.
\begin{lemma}
    \label{lemma: erm lower bound star term}
    For any $(w^g, w_1, \ldots, w_P)$, with probability $1-\frac{1}{T^2}$,
    \begin{align*}
         \left[\frac{1}{P}\sum_{t=1}^P \Delta_{i,t}(w^g, w_i)\right]_- \leq \otil\left( \sqrt{\frac{\overline{\sigma}^2 \overline{d}}{t}} + P\cdot\frac{\overline{d}}{t}\right).
    \end{align*}
\end{lemma}

\begin{proof}
    By the definition of $w^g_*$ and $w_{i,*}$, we have for all $w^g, w_i$,
    \begin{align*}
        \sum_{i=1}^P \left(L_i(w^g, w_i) - L_i(w^g_*, w_{i,*})\right) \geq 0.
    \end{align*}
    Then by Lemma~\ref{lemma: Lhat and L}, we have with probability at least $1-\frac{1}{T}$,
    \begin{align*}
        &\sum_{i=1}^P \Delta_{i,t}(w^g, w_i) \\
        &= \sum_{i=1}^P \left(\widehat{L}_{i,t}(w^g, w_i) - \widehat{L}_{i,t}(w^g_*, w_{i,*})\right) \\
        &\geq \sum_{i=1}^P \left(L_i(w^g, w_i) - L_i(w^g_*, w_{i,*})\right)  -\otil\left(P\cdot \sqrt{\frac{\overline{\sigma}^2 \overline{d}}{t-1}} + P^2\cdot\frac{\overline{d}}{t-1}\right) \\
        &= -\otil\left(P\cdot \sqrt{\frac{\overline{\sigma}^2 \overline{d}}{t-1}} + P^2\cdot\frac{\overline{d}}{t-1}\right).
    \end{align*}
\end{proof}

Finally, we are now able to prove Theorem~\ref{theorem: main for erm}. We provide a complete statement of the theorem below.

\begin{rtheorem}{Theorem}{\ref{theorem: main for erm}}
    Suppose the variance of the loss $\Var[\ell_{i,t}(w^g, w_i)]$ is upper bounded by $\sigma_i^2$, and  suppose $\sigma_i^2\leq \sigma^2$ for all $i$. Then
    \alg.ERM (Algorithm~\ref{alg: ERM client} and \ref{alg: ERM server}) guarantees
    \begin{align}
        &\E\left[ \frac{1}{PT}\sum_{i=1}^P \sum_{t=1}^T \left(\ell_{i,t}\left(w^g_{t-\down}, w_{i,t}\right) - \ell_{i,t}\Big(w^g_*, w_{i,*}\Big)\right)\right]\\
        &=\otil\left(\sqrt{\frac{\left(d+\sum_{i=1}^P d_i\right)\sigma^2}{PT}}
        + \frac{(1+D^2\gamma) \tau^{\frac{3}{4}}}{T^{\frac{3}{4}}}
        + \frac{(1+D^2\gamma) \tau + \left(d+\sum_{i=1}^P d_i\right)}{T} \right)
    \end{align}
\end{rtheorem}

\begin{proof}[Proof of Theorem~\ref{theorem: main for erm}]
Let $C_0=18\gamma D^2 + 2$. Combining Lemma~\ref{lemma: erm component 1} and \ref{lemma: erm component 2}, we get
that for $t> C_0\tau$


\begin{align*}
         &\E\left[\delta_{t-\down, t}\right] \\
         &\leq \E\left[\delta_{t-\down, t-\down-\up}\right] + \frac{\tau}{t} \times \otil\left( \frac{\overline{\sigma}\sqrt{\overline{d}}}{\sqrt{t-\tau}} + \frac{P\overline{d}}{t-\tau} \right) \tag{Lemma~\ref{lemma: erm component 2}} \\
         &\leq \E\left[\delta_{t-\tau-\down, t-\tau}\right] - \frac{1}{C_0} \E\left[\left[ \delta_{t-\tau-\down, t-\tau} \right]_+^2 \right]+  \frac{\tau}{t} \times \otil\left( \frac{\overline{\sigma}\sqrt{\overline{d}}}{\sqrt{t-\tau}} + \frac{P\overline{d}}{t-\tau} \right) \tag{Lemma~\ref{lemma: erm component 1}}\\
         &= \E\left[\delta_{t-\tau-\down, t-\tau}\right] - \frac{1}{C_0} \E\left[ \delta_{t-\tau-\down, t-\tau}^2 \right] +  \frac{\tau}{t} \times \otil\left( \frac{\overline{\sigma}\sqrt{\overline{d}}}{\sqrt{t-\tau}} + \frac{P\overline{d}}{t-\tau} \right) + \frac{1}{C_0}\E\left[\left[ \delta_{t-\tau-\down, t-\tau} \right]_-^2\right]\\
         &\leq \E\left[\delta_{t-\tau-\down, t-\tau}\right] - \frac{1}{C_0} \E\left[ \delta_{t-\tau-\down, t-\tau}^2 \right] +  \frac{\tau}{t} \times \otil\left( \frac{\overline{\sigma}\sqrt{\overline{d}}}{\sqrt{t-\tau}} + \frac{P\overline{d}}{t-\tau} \right) + \frac{1}{C_0}\times \otil\left( \frac{\overline{\sigma}^2 \overline{d}}{t-\tau} + \frac{P^2\overline{d}^2}{(t-\tau)^2}\right) \tag{Lemma~\ref{lemma: erm lower bound star term}}
\end{align*}

Now we focus on $t$'s that can be represented as $t=n\tau$ with integer $n$. Define  $B_n=\delta_{n\tau-\down, n\tau}$. Then the above implies
\begin{align*}
    B_n
    &\leq B_{n-1} - \frac{1}{C_0} B_{n-1}^2 +  \otil\left( \frac{\overline{\sigma}^2\overline{d}}{(n-1)C_0\tau}  + \frac{ \overline{\sigma}\sqrt{\overline{d}}}{(n-1)^{\frac{3}{2}}\sqrt{\tau}} + \frac{P \overline{d}}{(n-1)^2\tau} + \frac{P^2\overline{d}^2}{(n-1)^2 C_0\tau^2}\right).
\end{align*}
Define $C_1=\frac{1}{C_0}, C_2=\frac{\overline{\sigma}^2\overline{d}}{C_0\tau},
C_3=\frac{\overline{\sigma}\sqrt{\overline{d}}}{\sqrt{\tau}}, C_4=\frac{P\overline{d}}{\tau} + \frac{P^2\overline{d}^2}{C_0\tau^2}$. Then the above can be written as
\begin{align*}
    B_n\leq B_{n-1}-C_1 B_{n-1}^2 + \otil\left( \frac{C_2}{n-1} + \frac{C_3}{(n-1)^{\frac{3}{2}}} + \frac{C_4}{(n-1)^2}\right).
\end{align*}

Then using the Lemma~\ref{lemma: recursion} below, we have
\begin{align*}
    B_n
    &\leq \otil\left( \frac{\overline{\sigma}\sqrt{\overline{d}}}{\sqrt{n\tau}} \right) + \otil\left(\frac{C_0}{n^{\frac{3}{4}}}+\frac{\sqrt{C_0\overline{\sigma}}\cdot\overline{d}^{\frac{1}{4}}}{n^{\frac{3}{4}}\tau^{\frac{1}{4}}}\right)
    + \otil\left(\frac{C_0}{n} + \frac{\sqrt{C_0 P\overline{d}}}{n\sqrt{\tau}} + \frac{P\overline{d}}{n\tau}\right) \\
    &= \otil\left( \frac{\overline{\sigma}\sqrt{\overline{d}}}{\sqrt{n\tau}} \right) + \otil\left(\frac{C_0}{n^{\frac{3}{4}}}\right)
    + \otil\left(\frac{C_0}{n} +  \frac{P\overline{d}}{n\tau}\right)
    \tag{simplify the bound using $\frac{\overline{\sigma}\sqrt{\overline{d}}}{\sqrt{n\tau}} + \frac{C_0}{n^{\frac{3}{4}}} \geq 2\cdot\frac{\sqrt{C_0\overline{\sigma}}\cdot\overline{d}^{\frac{1}{4}}}{n^{\frac{3}{4}}\tau^{\frac{1}{4}}} $ and $\frac{C_0}{n}+\frac{P\overline{d}}{n\tau}\geq 2\cdot \frac{\sqrt{C_0 P\overline{d}}}{n\sqrt{\tau}}$}
\end{align*}

Replacing $n\tau$ back to $t$, we get
\begin{align}
    \E\left[\frac{1}{P}\sum_{i=1}^P\Delta_{i,t}(w^g_{t-\down}, w_{i,t})\right] = \otil\left(\frac{\overline{\sigma}\sqrt{\overline{d}}}{\sqrt{t}} \right) + \otil\left( \frac{C_0 \tau^{\frac{3}{4}}}{t^{\frac{3}{4}}} \right) + \otil\left( \frac{C_0 \tau + P\overline{d}}{t} \right). \label{eqn: replace back}
\end{align}
For $t=n\tau+1, \ldots, n\tau+(\tau-1)$, we can use the same approach to prove it. Thus, \eqref{eqn: replace back} actually holds for all $t>C_0\tau$. Finally, by Lemma~\ref{lemma: Lhat and L}, we have
\begin{align*}
    &\E\left[\frac{1}{P}\sum_{i=1}^P \left(\ell_{i,t}(w^g_{t-\down}, w_{i,t}) - \ell_{i,t}(w^g_{*}, w_{i,*})\right) \right] \\
    &= \E\left[\frac{1}{P}\sum_{i=1}^P \left(L_i(w^g_{t-\down}, w_{i,t}) - L_i(w^g_{*}, w_{i,*})\right) \right] \\
    &\leq \E\left[\frac{1}{P}\sum_{i=1}^P \left(\widehat{L}_{i,t}(w^g_{t-\down}, w_{i,t}) - \widehat{L}_{i,t}(w^g_{*}, w_{i,*})\right) \right] + \otil\left(\frac{\overline{\sigma}\sqrt{\overline{d}}}{t} + \frac{P\overline{d}}{t}\right) \\
    &=\E\left[\frac{1}{P}\sum_{i=1}^P \Delta_{i,t}(w^g_{t-\down}, w_{i,t}) \right] + \otil\left(\frac{\overline{\sigma}\sqrt{\overline{d}}}{t} + \frac{P\overline{d}}{t}\right).
\end{align*}
Combining this with \eqref{eqn: replace back}, and summing over $t>C_0\tau$ finish the proof.

\end{proof}

\begin{lemma}
    \label{lemma: recursion}
    Suppose $B_n\leq B_{n-1} - C_1 B_{n-1}^2 + \frac{C_2}{n-1} + \frac{C_3}{(n-1)^{\frac{3}{2}}} + \frac{C_4}{(n-1)^2}$ holds for all $n> n_0\geq 1$ with $C_1, C_2, C_3, C_4>0$, and $B_{n_0}\leq R$. Then for all $n\geq n_0$,
    \begin{align}
        B_n \leq \frac{D_1}{\sqrt{n}} + \frac{D_2}{n^{\frac{3}{4}}} + \frac{D_3}{n}.  \label{eqn: hypothesis}
    \end{align}
    where $D_1=\sqrt{\frac{2C_2}{C_1}}$, $D_2=\frac{1+\sqrt{1+2C_1(D_1+C_3)}}{C_1}$, $D_3=\frac{1+\sqrt{1+4C_1C_4}}{C_1}+n_0 R$.
\end{lemma}
\begin{proof}
    We use induction. When $n=n_0$, $B_{n_0}\leq R\leq \frac{D_3}{n_0}$ by our assumption. Suppose \eqref{eqn: hypothesis} holds for $n-1$, then
    \begin{align}
        B_n \leq \frac{D_1}{\sqrt{n-1}} + \frac{D_2}{(n-1)^{\frac{3}{4}}} + \frac{D_3}{n-1}- C_1\left(\frac{D_1^2}{n-1} + \frac{D_2^2}{(n-1)^{\frac{3}{2}}} + \frac{D_3^2}{(n-1)^2}\right) + \frac{C_2}{n-1} + \frac{C_3}{(n-1)^{\frac{3}{2}}} + \frac{C_4}{(n-1)^2}   \label{eqn: recurrsion bound}
    \end{align}
    where we use that for $a,b,c>0$, $(a+b+c)^2\geq a^2+b^2+c^2$. Now we prove that the right-hand side of \eqref{eqn: recurrsion bound} is upper bounded by $\frac{D_1}{\sqrt{n}} + \frac{D_2}{n^{\frac{3}{4}}} + \frac{D_3}{n}$. This is equivalent to
    \begin{align}
        &D_1\left(\frac{1}{\sqrt{n-1}}-\frac{1}{\sqrt{n}}\right) + D_2\left(\frac{1}{(n-1)^{\frac{3}{4}}} - \frac{1}{n^{\frac{3}{4}}}\right) + D_3\left(\frac{1}{n-1}-\frac{1}{n}\right) + \frac{C_2}{n-1} + \frac{C_3}{(n-1)^{\frac{3}{2}}} + \frac{C_4}{(n-1)^2} \nonumber \\
        &\leq C_1\left(\frac{D_1^2}{n} + \frac{D_2^2}{n^{\frac{3}{2}}} + \frac{D_3^2}{n^2}\right).  \label{eqn: tmptmpt}
    \end{align}
    Using the inequality $\frac{1}{(n-1)^k} - \frac{1}{n^k}\leq \frac{k}{n(n-1)^k }$ for $0\leq k\leq n$, we can bound left-hand side of \eqref{eqn: tmptmpt} by
    \begin{align*}
       &\frac{D_1}{n\sqrt{n-1}} + \frac{ D_2}{n(n-1)^{\frac{3}{4}}} +\frac{D_3}{n(n-1)} + \frac{C_2}{n-1} +  \frac{C_3}{(n-1)^{\frac{3}{2}}} + \frac{C_4}{(n-1)^2}
       \leq \frac{2C_2}{n} + \frac{2(D_1 + D_2 + C_3 )}{n^{\frac{3}{2}}} + \frac{2D_3+4C_4}{n^2}.
    \end{align*}
    Therefore, we only need to prove
    \begin{align*}
        2C_2 \leq C_1D_1^2, \qquad \qquad 2(D_1 + D_2 + C_3) \leq C_1D_2^2, \qquad \qquad 2D_3+4C_4 \leq C_1D_3^2.
    \end{align*}
    They are indeed satisfied by our choice of $D_1, D_2, D_3$.
\end{proof}

\section{Proofs for Theorem~\ref{theorem: main for sgd} (\alg.SGD algorithm)}
\label{sec:proof_sgd}
The complete statement of Theorem~\ref{theorem: main for sgd} is as follows. Note that as stated in Theorem~\ref{theorem: main for sgd}, the $\sigma_i$ is defined slightly different from that in Definition~\ref{definition: d sigma}. Also, note that our \alg.SGD can deal with more general cases than \alg.ERM in the sense that the delays $\up_i, \down_i$ can be different for different clients.

\begin{rtheorem}{Theorem}{\ref{theorem: main for sgd}}
Suppose the variance of the gradient of the losses of client $i$, $\Var[\nabla \ell_{i,t}(w^g, w_i)]$, is upper bounded by $\sigma_i^2$, and suppose $\sigma_i^2\leq \sigma^2$. Then \alg.SGD (Algorithm~\ref{alg: SGD client} and \ref{alg: SGD server}) guarantees that
    \begin{align}
        &\E\left[ \frac{1}{PT}\sum_{i=1}^P \sum_{t=1}^T \ell_{i,t}\left(w^g_{t-\down_i}, w_{i,t}\right) - \ell_{i,t}\Big(w^g_*, w_{i,*}\Big)\right] \nonumber \\
        &= \frac{1}{PT} \times \order\left( \frac{\|w_*^g\|^2}{\eta} + \sum_{i=1}^P \frac{\|w_{i,*}\|^2}{\eta_i} + \eta T\sum_{i=1}^P \sigma_i^2 + T\sum_{i=1}^P \eta_i\sigma_i^2 + \gamma \eta^2 P^2 G^2 T \sum_{i=1}^P \tau_i^2 + \gamma G^2T \sum_{i=1}^P \eta_i^2 \tau_i^2 + DG\sum_{i=1}^P \tau_i\right). \nonumber\\
     \end{align}
     Picking
     \begin{align*}
         \eta=\eta_i = \min\left\{\sqrt{\frac{\|w_*^g\|^2 + \sum_{i=1}^P \|w_{i,*}\|^2}{TP\sigma^2}}, \sqrt[3]{\frac{\|w_*^g\|^2 + \sum_{i=1}^P \|w_{i,*}\|^2}{\gamma P^3 G^2 \tau^2 T}}\right\},
     \end{align*}
     the above regret can be further upper bounded by
     \begin{align}
        \order\left(\sqrt{\frac{\left(\|w^g_*\|^2 + \sum_{i=1}^P \|w_{i,*}\|^2 \right)\sigma^2}{PT}} + \frac{\left(\gamma D^4 G^2 \tau^2\right)^{\frac{1}{3}}}{T^{\frac{2}{3}}} + \frac{DG\tau}{T}\right).  
    \end{align}
\end{rtheorem}
\begin{proof}[Proof of Theorem~\ref{theorem: main for sgd}]

The objective is
\begin{align*}
    &\E\left[\sum_{t=1}^T \sum_{i=1}^P \left(\ell_{i,t}(w_{t-\down_i}^g, w_{i,t}) - \ell_{i,t}(w^g_*, w_{i,*})\right)\right] \\
    &=\E\left[\sum_{t=1}^T \sum_{i=1}^P \left(L_i(w_{t-\down_i}^g, w_{i,t}) - L_i(w^g_*, w_{i,*})\right)\right] \\
    &\leq \E\Bigg[\underbrace{\sum_{t=1}^T \sum_{i=1}^P (w_{t-\down_i}^g-w^g_*)\cdot \nabla^g L_i(w_{t-\down_i-1}^g, w_{i,t-1})}_{\text{see Lemma~\ref{lemma: first term SGD}}}\Bigg] + \E\Bigg[\underbrace{\sum_{t=1}^T\sum_{i=1}^P (w_{i,t}-w_{i,*})\cdot \nabla^\ell L_i(w_{i,t-\down_i-1}^g, w_{i,t-1})}_{\text{see Lemma~\ref{lemma: SGD local term bound}}}\Bigg]
    \\
    &\qquad \qquad + \frac{\gamma}{2}\E\Bigg[\sum_{t=1}^T\sum_{i=1}^P\|w_{t-\down_i}^g - w_{t-\down_i-1}^g\|^2\Bigg] + \frac{\gamma}{2}\Bigg[\sum_{t=1}^T\sum_{i=1}^P\|w_{i,t}-w_{i,t-1}\|^2\Bigg].     \tag{by Lemma~\ref{lemma: useful}}
\end{align*}

By our update rules \eqref{eqn: local model update}, \eqref{eqn: global model update}, the third and the fourth terms above can be upper bounded by $\order(\gamma TP \sup_t \|w^g_t-w^g_{t-1}\|^2)=\order(\gamma TP (\eta PG)^2)$ and $\order(\gamma T\sum_{i=1}^P \eta_i^2G^2)$ respectively. Combining them with the following Lemma~\ref{lemma: first term SGD}, \ref{lemma: SGD local term bound}, and \ref{lemma: SGD unprocessed term}, we can bound the last expression by
\begin{align*}
    \order\left( \frac{\|w_*^g\|^2}{\eta} + \sum_{i=1}^P \frac{\|w_{i,*}\|^2}{\eta_i} + \eta T\sum_{i=1}^P \sigma_i^2 + T\sum_{i=1}^P \eta_i\sigma_i^2 + \gamma \eta^2 P^2 G^2 T \sum_{i=1}^P \tau_i^2 + \gamma G^2T \sum_{i=1}^P \eta_i^2 \tau_i^2 + DG\sum_{i=1}^P \tau_i\right).
\end{align*}
\end{proof}

The following two lemmas deal with two unprocessed terms in the proof of Theorem~\ref{theorem: main for sgd}.
\begin{lemma}
    \label{lemma: first term SGD}
    \begin{align*}
        &\E\Bigg[\sum_{t=1}^T \sum_{i=1}^P (w_{t-\down_i}^g-w^g_*)\cdot \nabla^g L_i(w_{t-\down_i-1}^g, w_{i,t-1})\Bigg] \\
        &\leq \frac{\|w_*^g\|^2}{2\eta} + \eta T\sum_{i=1}^P \sigma_i^2 +\E\left[\sum_{i=1}^P\sum_{t=1}^{T}   (w_{t-\down_i}^g- w_{t+\up_i-1}^g)\cdot  \nabla^g L_i(w_{t-\down_i-1}^g, w_{i,t-1})\right] + \order\left(DG\sum_{i=1}^P \tau_i\right).
    \end{align*}
\end{lemma}

\begin{proof}
\begin{align}
    &\sum_{t=1}^T \sum_{i=1}^P (w_{t-\down_i}^g-w^g_*)\cdot \nabla^g L_i(w_{t-\down_i-1}^g, w_{i,t-1})   \nonumber \\
    &= \sum_{t=1}^T (w^g_t-w_*^g)\cdot \underbrace{ \sum_{i=1}^P \nabla^g L_i(w_{t-1}^g, w_{i,t+\down_i-1}) }_{a_t} + \order\left(DG\sum_{i=1}^P \down_{i} \right)   \nonumber \\
    &= \sum_{t=1}^T (w_t^g-w_*^g)\cdot \underbrace{ \left( \sum_{i=1}^P  \nabla^g_{i,t-\up_i}  \right) }_{b_t} + \sum_{t=1}^T (w_t^g - w^g_*)\cdot (a_t-b_t) + \order\left(DG\sum_{i=1}^P \down_{i} \right)\nonumber  \\
    &\leq \sum_{t=1}^T \frac{\|w_*^g - w_{t-1}^g\|^2 - \|w_*^g - w_t^g\|^2-\|w^g_{t-1}-w^g_{t}\|^2}{2\eta}  + \underbrace{\sum_{t=1}^T (w_t^g - w^g_*)\cdot (a_{t}-b_{t})}_{\term_1} + \order\left(DG\sum_{i=1}^P \down_{i} \right).  \label{eqn: first decompose}
\end{align}
Note that $b_t$ is the gradient that is used to update the global model from $w^g_{t-1}$ to $w^g_{t}$ (Eq.\eqref{eqn: global model update}). Therefore using Lemma~\ref{lemma: gradient update lemma} we have the last equality.

We continue to bound $\term_1$. We use $c_t $ to denote the expectation of $b_t$ conditioned on all examples that reach the server before time $t$. That is,
\begin{align}
    c_t
    &= \E\Big[b_t~\Big|~ \ell_{i,s}: s<t-\up_i\Big] \nonumber \\
    &= \E\left[\sum_{i=1}^P \Delta_{i,t-\up_i}^g ~\Big|~ \ell_{i,s}: s<t-\alpha_i\right] \nonumber \\
    &= \sum_{i=1}^P \E\left[\nabla^g \ell_{i,t-\up_i}(w^g_{t-\up_i-\down_i}, w_{i,t-\up_i})~\Big|~ \ell_{i,s}: s<t-\alpha_i \right]  \nonumber \\
    &= \sum_{i=1}^P \nabla^g L_i(w^g_{t-\up_i-\down_i}, w_{i,t-\up_i}).
    \label{eqn: interpret c tau}
\end{align}
The last equality comes from the fact that $w^g_{t-\down_i-\up_i}$ and $w_{i,t-\up_i}$ only depend on $\ell_{i,s}$ with $s<t-\up_i$ (see update rules \eqref{eqn: local model update}, \eqref{eqn: global model update}).
Then we can decompose $\term_1$ as follows:
\begin{align}
    \term_1
    &= \sum_{t=1}^T (w_t^g - w^g_*)\cdot (a_{t}-b_{t}) \nonumber \\
    &= \sum_{t=1}^T (w_t^g - w^g_*)\cdot (a_{t}-c_{t}) + \sum_{t=1}^T (w_{t-1}^g-w_*^g)\cdot(c_t-b_t) + \sum_{t=1}^T (w^g_t-w^g_{t-1})\cdot(c_t-b_t).  \label{eqn: term 1 decomposition}
\end{align}
Since $w^g_{t-1}$ only depends on $\{\ell_{i,s}: s< t-\up_i\}$ (by Algorithm~\ref{alg: SGD server}), the conditional expectation of the second term in \eqref{eqn: term 1 decomposition} is
\begin{align}
    \E\left[\sum_{t=1}^T (w_{t-1}^g-w_*^g)\cdot(c_t-b_t)~\Bigg|~\ell_{i,s}: s< t-\up_i\right]
    = \sum_{t=1}^T (w_{t-1}^g-w_*^g)\cdot\E\Big[c_t-b_t~\Big|~\ell_{i,s}: s< t-\up_i\Big]=0
    \label{eqn: further decompose 1}
\end{align}
by Eq.\eqref{eqn: interpret c tau}.
The third term in \eqref{eqn: term 1 decomposition} can be bounded as
\begin{align}
    \sum_{t=1}^T (w^g_t-w^g_{t-1})\cdot(c_t-b_t) \leq \sum_{t=1}^T \frac{\|w^g_t-w^g_{t-1}\|^2}{4\eta} + \eta \sum_{t=1}^T \|b_t-c_t\|^2. \label{eqn: further decompose 2}
\end{align}
Observe that $\E\left[\|b_t-c_t\|^2~|~\ell_{i,s}:s<t-\up_i\right] = \Var\left[b_t~|~\ell_{i,s}:s<t-\up_i\right]$. By the independence among the examples from different clients, we can bound
\begin{align}
    \E\left[\sum_{t=1}^T \|b_t-c_t\|^2\right] = \sum_{t=1}^T\sum_{i=1}^P \Var[\nabla^g_{i,t-\up_i}] \leq T\sum_{i=1}^P \sigma_i^2. \label{eqn: global bound covariance}
\end{align}
Now we deal with the first term in \eqref{eqn: term 1 decomposition}:
\begin{align}
    &\sum_{t=1}^T (w_t^g - w_*^g) \left( \sum_{i=1}^P   \nabla^g L_i(w_{t-1}^g, w_{i,t+\down_i-1}) -  \sum_{i=1}^P \nabla^g L_i(w^g_{t-\up_i-\down_i}, w_{i,t-\up_i}) \right) \nonumber   \\
    &= \sum_{i=1}^P\sum_{t=1}^{T}   (w_{t-\down_i}^g- w_{t+\up_i-1}^g)\cdot  \nabla^g L_i(w_{t-\down_i-1}^g, w_{i,t-1}) + \order\left(DG\sum_{i=1}^P \tau_i\right). \tag{re-indexing}  \\
    \label{eqn: further decompose 3}
\end{align}

Combining Eq.\eqref{eqn: first decompose}-\eqref{eqn: further decompose 3}, we see that the right-hand side of \eqref{eqn: first decompose}, after taking expectation, is upper bounded by
\begin{align*}
    &\frac{\|w_{0}^g-w_*^g \|^2}{2\eta} + \sum_{t=1}^T  \left(\frac{-1}{2\eta} + \frac{1}{4\eta}\right)\E\left[\|w_t^g-w_{t-1}^g\|^2 \right] + \eta T\sum_{i=1}^P \sigma_i^2 \\
    &\qquad + \E\left[\sum_{i=1}^P\sum_{t=1}^{T}   (w_{t-\down_i}^g- w_{t+\up_i-1}^g)\cdot  \nabla^g L_i(w_{t-\down_i-1}^g, w_{i,t-1})\right] + \order\left(DG\sum_{i=1}^P \tau_i\right)\\
    &\leq \frac{\|w_*^g\|^2}{2\eta} + \eta T\sum_{i=1}^P \sigma_i^2 + \E\left[\sum_{i=1}^P\sum_{t=1}^{T}   (w_{t-\down_i}^g- w_{t+\up_i-1}^g)\cdot  \nabla^g L_i(w_{t-\down_i-1}^g, w_{i,t-1})\right] + \order\left(DG\sum_{i=1}^P \tau_i\right).
\end{align*}
\end{proof}

\begin{lemma}
    \label{lemma: SGD local term bound}
    \begin{align*}
        &\E\left[\sum_{t=1}^T\sum_{i=1}^P (w_{i,t}-w_{i,*})\cdot \nabla^l L_i(w_{i,t-\down_i-1}^g, w_{i,t-1})\right] \\
        &\leq \sum_{i=1}^P \frac{\|w_{i,*}\|^2}{2\eta_i} + T\sum_{i=1}^P \eta_i \sigma_i^2  + \E\left[\sum_{t=1}^{T} \sum_{i=1}^P (w_{i,t}-w_{i,t+\down_i+\up_i-1})\cdot \nabla^\ell L_i(w_{i,t-\down_i-1}^g, w_{i,t-1})\right] + \order\left(DG\sum_{i=1}^P \tau_i\right).
    \end{align*}
\end{lemma}

\begin{proof}
This proof goes through almost the same procedure as in Lemma~\ref{lemma: first term SGD}'s proof.
\begin{align}
    &\sum_{t=1}^T\sum_{i=1}^P (w_{i,t}-w_{i,*})\cdot \underbrace{\nabla^l L_i(w_{i,t-\down_i-1}^g, w_{i,t-1})}_{d_{i,t}}  \nonumber  \\
    &= \sum_{t=1}^T\sum_{i=1}^P (w_{i,t}-w_{i,*})\cdot \underbrace{\nabla^l_{i,t-\down_i-\up_i}}_{e_{i,t}} + \sum_{t=1}^T\sum_{i=1}^P  (w_{i,t}-w_{i,*})\cdot \left(d_{i,t} - e_{i,t}\right)  \nonumber \\
    &\leq \sum_{t=1}^T \sum_{i=1}^P \frac{\|w_{i,t-1}-w_{i,*}\|^2 - \|w_{i,t}-w_{i,*}\|^2 - \|w_{i,t-1}-w_{i,t} \|^2}{2\eta_i} + \underbrace{\sum_{t=1}^T\sum_{i=1}^P  (w_{i,t}-w_{i,*})\cdot \left(d_{i,t} - e_{i,t}\right)}_{\term_2}. \label{eqn: local part bound}
\end{align}
The last inequality is by Lemma~\ref{lemma: gradient update lemma} and the fact that $e_{i,t}$ is the gradient that is used to update the local model from $w_{i,t-1}$ to $w_{i,t}$.
To bound $\term_2$, we define
\begin{align*}
    f_{i,t}
    &=\E[e_{i,t}~|~\ell_{i,s}: s< t-\down_i-\up_i ]\\
    &=  \E\left[\nabla^l \ell_{i,t-\down_i-\up_i}(w^g_{t-2\down_i-\up_i}, w_{i,t-\down_i-\up_i})~\Big|~\ell_{i,s}:s<t-\down_i-\up_i\right] \\
    &= \nabla^l L_i(w^g_{t-2\down_i-\up_i}, w_{i,t-\down_i-\up_i})
\end{align*}
because $w^g_{t-2\down_i-\up_i}$ and $w_{i,t-\down_i-\up_i}$ only depend on $\ell_{i,s}$ with $s<t-\down_i-\up_i$.
Then we make the following decomposition:
\begin{align}
    \term_2
    &= \sum_{t=1}^T\sum_{i=1}^P  (w_{i,t}-w_{i,*})\cdot (d_{i,t}-f_{i,t})  + \sum_{t=1}^T\sum_{i=1}^P (w_{i,t-1}-w_{i,*})\cdot (f_{i,t}-e_{i,t}) + \sum_{t=1}^T\sum_{i=1}^P (w_{i,t}-w_{i,t-1})\cdot (f_{i,t}-e_{i,t}).   \label{eqn: term 2 decomposition}
\end{align}
The second term in \eqref{eqn: term 2 decomposition} has zero expectation because
\begin{align}
    &\E[(w_{i,t-1}-w_{i,*})\cdot(f_{i,t}-e_{i,t})~|~\ell_{i,s}:s< t-\down_i-\up_i]   \nonumber \\
    &=(w_{i,t-1}-w_{i,*})\cdot \E[(f_{i,t}-e_{i,t})~|~\ell_{i,s}:s< t-\down_i-\up_i] =0.
\end{align}

The third term in \eqref{eqn: term 2 decomposition} can be upper bounded as
\begin{align}
    \sum_{t=1}^T\sum_{i=1}^P (w_{i,t}-w_{i,t-1})\cdot (f_{i,t}-e_{i,t}) \leq \sum_{t=1}^T \sum_{i=1}^P \left(\frac{\|w_{i,t}-w_{i,t-1}\|^2}{4\eta_i} + \eta_i\|f_{i,t}-e_{i,t}\|^2\right),
\end{align}
and we note that $\E[\|f_{i,t}-e_{i,t}\|^2~|~\ell_{i,s}:s<t-\down_i-\up_i] = \Var[e_{i,t}~|~\ell_{i,s}:s<t-\down_i-\up_i]$ is the conditional variance of $e_{i,t}$. Since all samples are independent, we can bound
\begin{align}
    &\E\left[\sum_{t=1}^T\|f_{i,t}-e_{i,t}\|^2~\Bigg|~\ell_{i,s}: s<t-\down_i-\up_i \right]\nonumber  \leq \sum_{t=1}^T \Var\left[\nabla_{i,t-\down_i-\up_i}~\Bigg|~\ell_{i,s}: s<t-\down_i-\up_i \right] \leq   T\sigma_i^2.
\end{align}

The first term in \eqref{eqn: term 2 decomposition} is
\begin{align}
    &\sum_{t=1}^T \sum_{i=1}^P (w_{i,t}-w_{i,*})\cdot (d_{i,t}-f_{i,t})   \nonumber \\
    &= \sum_{t=1}^T \sum_{i=1}^P (w_{i,t}-w_{i,*})\cdot \left(\nabla^l L_i(w_{i,t-\down_i-1}^g, w_{i,t-1}) -  \nabla^l L_i(w^g_{t-2\down_i-\up_i}, w_{i,t-\down_i-\up_i})\right)   \nonumber \\
    &= \sum_{t=1}^{T} \sum_{i=1}^P (w_{i,t}-w_{i,t+\down_i+\up_i-1})\cdot \nabla^\ell L_i(w_{i,t-\down_i-1}^g, w_{i,t-1}) + \order\left(DG\sum_{i=1}^P \tau_i\right). \tag{telescoping and reindexing} \\
    & \label{eqn: the smooth term local}
\end{align}

Combining \eqref{eqn: local part bound}-\eqref{eqn: the smooth term local}, we get that the left-hand side of \eqref{eqn: local part bound}, after taking expectation, is upper bounded by
\begin{align*}
    &\sum_{i=1}^P \frac{\|w_{i,0}-w_{i,*}\|^2}{2\eta_i} + \sum_{t=1}^T \sum_{i=1}^P \left(-\frac{1}{2\eta_i} + \frac{1}{4\eta_i}\right)\E[\|w_{i,t}-w_{i,t-1}\|^2] + T\sum_{i=1}^P \eta_i \sigma_i^2 \\
    &\qquad + \E\left[\sum_{t=1}^{T} \sum_{i=1}^P (w_{i,t}-w_{i,t+\down_i+\up_i-1})\cdot \nabla^\ell L_i(w_{i,t-\down_i-1}^g, w_{i,t-1})\right] + \order\left(DG\sum_{i=1}^P \tau_i\right) \\
    &\leq \sum_{i=1}^P \frac{\|w_{i,*}\|^2}{2\eta_i} + T\sum_{i=1}^P \eta_i \sigma_i^2  + \E\left[\sum_{t=1}^{T} \sum_{i=1}^P (w_{i,t}-w_{i,t+\down_i+\up_i-1})\cdot \nabla^\ell L_i(w_{i,t-\down_i-1}^g, w_{i,t-1})\right] + \order\left(DG\sum_{i=1}^P \tau_i\right).
\end{align*}
\end{proof}

The following lemma further deals with the unprocessed terms in Lemma~\ref{lemma: first term SGD} and Lemma~\ref{lemma: SGD local term bound}.
\begin{lemma}
    \label{lemma: SGD unprocessed term}
    \begin{align*}
       &\E\left[ \sum_{t=1}^{T}\sum_{i=1}^P   (w_{t-\down_i}^g- w_{t+\up_i-1}^g)\cdot  \nabla^g L_i(w_{t-\down_i-1}^g, w_{i,t-1})
       + \sum_{t=1}^{T} \sum_{i=1}^P (w_{i,t}-w_{i,t+\down_i+\up_i-1})\cdot \nabla^l L_i(w_{i,t-\down_i-1}^g, w_{i,t-1})
       \right] \\
       &=\order\left( \gamma \eta^2 P^2 G^2 T \sum_{i=1}^P \tau_i^2 + \gamma G^2T \sum_{i=1}^P \eta_i^2 \tau_i^2 + DG\sum_{i=1}^P \tau_i\right).
    \end{align*}
\end{lemma}
\begin{proof}
    Define the joint parameter $u_{i,t}=(w^g_{t-\down_i}, w_{i,t})$. Then the left-hand side can be written as
    \begin{align*}
        \sum_{t=1}^T \sum_{i=1}^P (u_{i,t} - u_{i,t+\down_i+\up_i-1})\cdot \nabla L_i(u_{i,t-1}) = \sum_{t=1}^T \sum_{i=1}^P (u_{i,t} - u_{i,t+\tau_i-1})\cdot \nabla L_i(u_{i,t-1}).
    \end{align*}
    By Lemma~\ref{lemma: useful}, we can bound it by
    \begin{align*}
        \sum_{t=1}^T \sum_{i=1}^P (u_{i,t} - u_{i,t+\tau_i-1})\cdot \nabla L_i(i,u_{t-1})
        &\leq \sum_{t=1}^T \sum_{i=1}^P \left( L_i(u_{i,t})-L_i(u_{i,t+\tau_i-1}) + \frac{\gamma}{2} \|u_{i,t+\tau_i-1}-u_{i,t}\|^2 \right)\\
        & = \frac{\gamma}{2}\sum_{t=1}^T \sum_{i=1}^P \|u_{i,t+\tau_i-1}-u_{i,t}\|^2 + \order\left(DG\sum_{i=1}^T \tau_i\right).
    \end{align*}
    By our update rule, we have $\|w^g_{t+\tau_i-1}-w^g_t\|^2 \leq (\eta \tau_i PG)^2$ (from $t$ to $t+\tau_i-1$, there are $\tau_iP$ gradient updates for $w^g$) and $\|w_{i,t+\tau_i-1}-w_{i,t}\|^2 \leq (\eta_i\tau_i G)^2$. Combining them finishes the proof.
\end{proof}

\begin{lemma}
    \label{lemma: gradient update lemma}
    Let $w'=\Pi_{\Omega}(w-\eta g)$, where $\Pi_\Omega: \mathbb{R}^d\rightarrow \mathbb{R}^d$ is the projection operator that projects the input vector to the convex set $\Omega\subset \mathbb{R}^d$, and $w\in\Omega$, $g\in\mathbb{R}^d$, $\eta>0$. Then we have for any $w_*\in \Omega$,
    \begin{align*}
        (w'-w_*)\cdot g \leq \frac{\|w-w_*\|^2-\|w'-w_*\|^2-\|w'-w\|^2}{2\eta}.
    \end{align*}
\end{lemma}
\begin{proof}
    By the definition of $w'$, it is the minimizer of $\|w'-w+\eta g\|^2$ over $\Omega$. Therefore, by the first-order optimality condition, we have for any $w_*\in \Omega$,
    \begin{align*}
        (w'-w+\eta g)\cdot(w'-w_*)\leq 0.
    \end{align*}
    Rearranging it we get
    \begin{align*}
        (w'-w_*)\cdot g \leq \frac{(w-w')\cdot(w'-w_*)}{\eta} = \frac{\|w-w_*\|^2-\|w'-w_*\|^2-\|w'-w\|^2}{2\eta},
    \end{align*}
    where the last equality can be obtained by direct expansion.
\end{proof}

\begin{lemma}
    \label{lemma: useful}
    For any $\gamma$-smooth convex function $f$, and any $a,b,c$,
    \begin{align*}
        f(a) - f(b) \leq (a-b)\cdot\nabla f(c) + \frac{\gamma}{2}\|a-c\|^2.
    \end{align*}
\end{lemma}
\begin{proof}
    By the convexity and the $\gamma$-smoothness of $f$, we have
    \begin{align*}
        f(c)-f(b)&\leq (c-b)\cdot\nabla f(c), \\
        f(a)-f(c)&\leq (a-c)\cdot\nabla f(c) + \frac{\gamma}{2}\|a-c\|^2.
    \end{align*}
    Adding up two inequalities we get the desired inequality.
\end{proof}

\section{The Failure of the Fictitious-Play Variant of the ERM Algorithm}
\label{sec: comparing three}
In this section, we experimentally compare \alg.SGD (Algorithm~\ref{alg: SGD client},~\ref{alg: SGD server}), \alg.ERM (Algorithm~\ref{alg: ERM client},~\ref{alg: ERM server}), and the \emph{fictitious play} variant of the ERM algorithm that we describe at Eq.\eqref{eqn: fp client} and \eqref{eqn: fp server}. The goal is to show that the last one may take significantly more rounds to converge.

\subsection{Data Generation}
Suppose there is only one client. The feature dimensions are $2$ for both global and local features. The feature vectors $(x^g_t, x^l_t)$ and the label $y_t$ are generated i.i.d. according to
\begin{align*}
    a_t&\sim \mathcal{N}(0,1)\\
    b_t&\sim \mathcal{N}(0,1)\\
    x_t^g &= \begin{bmatrix}
        a_t + \epsilon_t\\
        b_t
    \end{bmatrix} \qquad \text{where $\epsilon_t\sim \mathcal{N}(0,0.25)$} \\
    x_t^l &= \begin{bmatrix}
        1-a_t\\
        1-b_t
    \end{bmatrix}\\
    y_t &= 1
\end{align*}
The loss is defined as $\ell_t(w^g, w^l) = (y_t-w^{g}\cdot x^g_t - w^{l}\cdot x^\ell_t)^2$.
Clearly, the best pair of regressors is $w^g_* = \begin{bmatrix}0 \\ 1\end{bmatrix}$, $w^l_* = \begin{bmatrix}0 \\ 1\end{bmatrix}$, and this pair gives zero average loss.
We run three algorithms for $T=20000$ steps.

\subsection{Algorithms}
We let the parameters be initialized as $w^g_1 = \begin{bmatrix}1\\0\end{bmatrix}, w^l_1 = \begin{bmatrix}1\\0\end{bmatrix}$. Then the goal of the algorithms is to adjust both $w^g_t$ and $w^l_t$ from $ \begin{bmatrix}1\\0\end{bmatrix}$ to $ \begin{bmatrix}0\\1\end{bmatrix}$ since the latter is the optimal solution.

Assume no delays. Then the three algorithms we compare can be simplified as in Algorithm~\ref{alg: FP_SGD}, \ref{alg: FP_ERM}, \ref{alg: FP_FP}. The main difference between Algorithm~\ref{alg: FP_ERM} and \ref{alg: FP_FP} is that in the former, the server (client) re-applies the new parameters from the client (server) to the old samples, but the latter does not. As we mentioned in Section~\ref{subsec: erm}, in terms of computational and communication efficiency, Algorithm~\ref{alg: FP_FP} is actually preferred over Algorithm~\ref{alg: FP_ERM}.
\begin{algorithm}[H]
\caption{\alg.SGD}
    \label{alg: FP_SGD}
    Let $\eta=1.0$ (an arbitrary choice). \\
    \For{$t=1, \ldots, T$}{
        Suffer loss $\ell_t(w_t^g, w_t^l)$ and make updates:
        \begin{align*}
            w^g_{t+1} = w_t^g - \eta \nabla^g \ell_t(w^g_t, w^l_t)\\
            w^l_{t+1} = w^l_t - \eta \nabla^l \ell_t(w^g_t, w^l_t)
        \end{align*}
    }
\end{algorithm}
\begin{algorithm}[H]
\caption{\alg.ERM}
    \label{alg: FP_ERM}
    \For{$t=1, \ldots, T$}{
        Suffer loss $\ell_t(w_t^g, w_t^l)$ and make updates:
        \begin{align*}
            w^g_{t+1} &= \argmin_{w^g} \sum_{s=1}^t \ell_s(w^g, w^l_t) \\
            w^l_{t+1} &= \argmin_{w^l} \sum_{s=1}^t \ell_s(w^g_{t}, w^l)
        \end{align*}
    }
\end{algorithm}
\begin{algorithm}[H]
\caption{Fictitious Play}
    \label{alg: FP_FP}
    \For{$t=1, \ldots, T$}{
        Suffer loss $\ell_t(w_t^g, w_t^l)$ and make updates:
        \begin{align*}
            w^g_{t+1} &= \argmin_{w^g} \sum_{s=1}^t \ell_s(w^g, w^l_s) \\
            w^l_{t+1} &= \argmin_{w^l} \sum_{s=1}^t \ell_s(w^g_s, w^l)
        \end{align*}
    }
\end{algorithm}

\subsection{Comparing the performance}
We compare the average loss performances of the three algorithms over time, and observe that the Fictitious-play strategy is highly sub-optimal (Figure~\ref{Fig: loss compare}). All plots in this section are an average over $50$ random rollouts.
\begin{figure}[H]
\includegraphics[width=10cm]{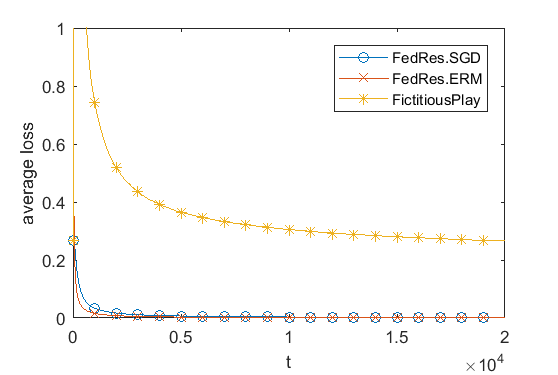}
\caption{Comparing the average loss performance among \alg.SGD (Algorithm~\ref{alg: FP_SGD}), \alg.ERM (Algorithm~\ref{alg: FP_ERM}) and the fictitious-play strategy (Algorithm~\ref{alg: FP_FP})}
\label{Fig: loss compare}
\end{figure}

\begin{figure}[H]
\includegraphics[width=8cm]{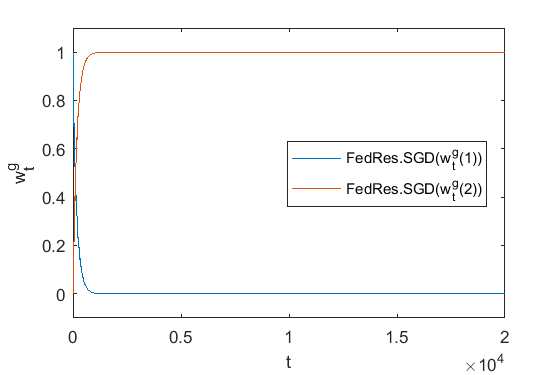}
\includegraphics[width=8cm]{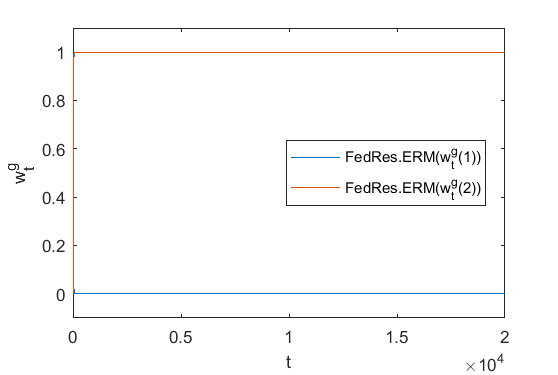}
\includegraphics[width=8cm]{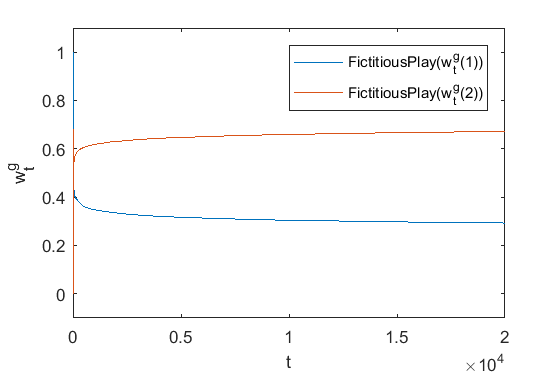}
\caption{The change of $w^g_t$ over time (better viewed with color). Each sub-figure is for one algorithm. The blue lines plot the first component of $w^g_t$, which is initialized as $1$ and the learner should adjust it to $0$; the red lines plot the second component of $w^g_t$, which is initialized as $0$ and should be adjusted to $1$. }
\label{fig: weight change}
\end{figure}

Recall that the goal of the algorithms is to change both $w^g_t$ and $w^l_t$ from $ \begin{bmatrix}1\\0\end{bmatrix}$ to $ \begin{bmatrix}0\\1\end{bmatrix}$. We plot the changes of the components of $w^g_t$ over time for three algorithms in Figure~\ref{fig: weight change}. From Figure~\ref{fig: weight change} we see that while \alg.SGD and \alg.ERM can quickly find the optimal solutions, the fictitious-play strategy gets stuck before reaching the optimum. Our explanation for this phenomenon is below. Observe that by our construction of $(x_t^g, x_t^l, y_t)$, if $w^g_t$ is of the form $\begin{bmatrix}z\\1-z\end{bmatrix}$ (e.g., in the beginning, $z$ is $1$), then it creates a loss for $w^l_t$ as
\begin{align*}
     &\left(1- \begin{bmatrix}z\\1-z\end{bmatrix}\cdot \begin{bmatrix}a_t\\b_t\end{bmatrix} - \begin{bmatrix}w^l_t(1)\\w^l_t(2)\end{bmatrix}\cdot  \begin{bmatrix}1-a_t\\1-b_t\end{bmatrix}\right)^2
     = \bigg((1-w_t^l(1)-w_t^l(2)) + a_t(w_t^l(1)-z) + b_t(w_t^l(2)-1+z)\bigg)^2,
\end{align*}
whose expectation is minimized when $w_t^l(1)=z$ and $w_t^l(2)=1-z$; that is, when $w_t^l=w_t^g$, the expected loss is minimized. Similarly, when $w_t^l$ is fixed, the expected loss minimizer for $w_t^g$ is $w_t^l$. Since the fictitious-play strategy memorizes all previous losses under the outdated parameters, $w_t^g$ tends to be close to the average of $w_s^l$'s with $s<t$; similarly, $w_t^l$ tends to be close to the average of previous $w_s^g$'s. Therefore, the server and the client tend to lock each other, and this makes their updates very slow, which results in the learning curve of the fictitious-play strategy that we observe in Figure~\ref{fig: weight change}.

\section{Proof for Theorem~\ref{theorem: cb}}
\label{section: proof of Theorem CB}

We provide the complete statement of Theorem~\ref{theorem: cb} below.
\begin{rtheorem}{Theorem}{\ref{theorem: cb}}
     With the algorithm stated in Section~\ref{subsec: epsilon-greedy for CB}, and supposed that $\|w_{i,*}\|$ are all upper bounded by $\|w_*^l\|$, the regret can be upper bounded as follows:
     \begin{align*}
          &\E\left[ \frac{1}{PT}\sum_{t=1}^T \sum_{i=1}^P r_{i,t}(a_{i,t}^*) - r_{i,t}(a_{i,t}) \right] \\
          &=  \mathcal{O}\left(\left(\frac{K^4\left(\|w_*^g\|^2 + \sum_{i=1}^P \|w_{i,*}\|^2 \right)\sigma^2}{PT}\right)^{\frac{1}{5}} + \frac{\left(K^6 \gamma D^4 G^2\right)^{\frac{1}{8}}}{T^{\frac{1}{4}}} + + \frac{K(\gamma D^4 G^2 \tau^2)^{\frac{1}{6}}+ (K^2 DG)^{\frac{1}{3}}}{T^{\frac{1}{3}}} + \frac{K\sqrt{DG\tau}}{\sqrt{T}}\right).
     \end{align*}
\end{rtheorem}

\begin{proof}
Below we derive the regret bound using the theorem for \alg.SGD (Theorem~\ref{theorem: main for sgd}). Since the update of model parameters are only once per $B$ rounds, the equivalent delay for client $i$ is $\lceil \frac{\tau_i}{B} \rceil$. Using Theorem~\ref{theorem: main for sgd}, we have the following bound:
\begin{align*}
    &\E\left[\frac{B}{PT} \sum_{i=1}^P\sum_{t=1}^{T} \left(\ell_{i,t}\left(\widehat{w}^g_{t}, \widehat{w}_{i,t}\right) - \ell_{i,t}(w^g_*, w_{i,*})\right)\one[t=nB] \right]\\
    &=\mathcal{O}\left(\sqrt{\frac{\left(\|w_*^g\|^2 + \sum_{i=1}^P \|w_{i,*}\|^2\right) \sigma^2}{P\cdot\frac{T}{B}}} + \frac{(\gamma D^4G^2\lceil\frac{\tau}{B}\rceil^2)^{\frac{1}{3}}}{\left(\frac{T}{B}\right)^{\frac{2}{3}}} + \frac{DG\lceil\frac{\tau}{B}\rceil}{\frac{T}{B}} \right)  
\end{align*}
where for simplicity, we assume $\|w_{i,*}\|^2\leq \|w_*^l\|^2$ for all $i$.
By the definition of $\ell_{i,t}$ and the realizability assumption,
\begin{align*}
     &\E\bigg[\bigg(\ell_{i,t}\left(\widehat{w}^g_{t}, \widehat{w}_{i,t}\right) - \ell_{i,t}(w^g_*, w_{i,*})\bigg)\one[t=nB]\bigg] \\
     &=\E\bigg[\bigg(2r_{i,t}(a_{i,t})-f(x_{i,t}(a_{i,t}); w^g_*, w_{i,*})-f(x_{i,t}(a_{i,t}); \widehat{w}^g_t, \widehat{w}_{i,t})\bigg)
     \bigg(f(x_{i,t}(a_{i,t}); w^g_*, w_{i,*})-f(x_{i,t}(a_{i,t}); \widehat{w}^g_t, \widehat{w}_{i,t})\bigg)\one[t=nB] \bigg] \\
     &=\E\left[\bigg(f(x_{i,t}(a_{i,t}); w^g_*, w_{i,*})-f(x_{i,t}(a_{i,t}); \widehat{w}^g_t, \widehat{w}_{i,t})\bigg)^2\one[t=nB]\right]\\
     &=\E\left[\frac{1}{K}\sum_{a=1}^K\bigg(f(x_{i,t}(a); w^g_*, w_{i,*})-f(x_{i,t}(a); \widehat{w}^g_t, \widehat{w}_{i,t})\bigg)^2\one[t=nB]\right]
\end{align*}

Therefore,
\begin{align*}
    &\E\left[\frac{B}{PKT} \sum_{i=1}^P\sum_{t=1}^{T} \sum_{a=1}^K\bigg(f(x_{i,t}(a); w^g_*, w_{i,*})-f(x_{i,t}(a); \widehat{w}^g_t, \widehat{w}_{i,t})\bigg)^2\one[t=nB] \right]\\
    &=\mathcal{O}\left(\sqrt{\frac{\left(\|w_*^g\|^2 + \sum_{i=1}^P \|w_{i,*}\|^2\right) \sigma^2 B}{PT}} + \frac{\left(\gamma D^4 G^2(\tau+B)^2\right)^{\frac{1}{3}}}{T^\frac{2}{3}} + \frac{DG(\tau + B)}{T}\right)
\end{align*}
Due to the i.i.d. assumption, the left-hand side is identical to
\begin{align*}
    \E\left[\frac{1}{PKT} \sum_{i=1}^P\sum_{t=1}^{T} \sum_{a=1}^K\bigg(f(x_{i,t}(a); w^g_*, w_{i,*})-f(x_{i,t}(a); \widehat{w}^g_t, \widehat{w}_{i,t})\bigg)^2\right].
\end{align*}
By Cauchy-Schwarz's inequality,
\begin{align*}
     &\frac{1}{PKT}\sum_{i=1}^P\sum_{t=1}^{T} \sum_{a=1}^K \bigg\vert f(x_{i,t}(a); w^g_*, w_{i,*})-f(x_{i,t}(a); \widehat{w}^g_t, \widehat{w}_{i,t}) \bigg\vert\\
     &\leq \frac{1}{PKT}\left(\sum_{i=1}^P\sum_{t=1}^{T} \sum_{a=1}^K \bigg( f(x_{i,t}(a); w^g_*, w_{i,*})-f(x_{i,t}(a); \widehat{w}^g_t, \widehat{w}_{i,t})\bigg)^2\right)^{\frac{1}{2}}\left(PKT\right)^{\frac{1}{2}}\\
     &= \left(\frac{1}{PKT}\sum_{i=1}^P\sum_{t=1}^{T} \sum_{a=1}^K \bigg( f(x_{i,t}(a); w^g_*, w_{i,*})-f(x_{i,t}(a); \widehat{w}^g_t, \widehat{w}_{i,t})\bigg)^2\right)^{\frac{1}{2}}.
\end{align*}
Combining them, we get
\begin{align}
     &\E\left[\frac{1}{PKT}\sum_{i=1}^P\sum_{t=1}^{T} \sum_{a=1}^K \bigg\vert f(x_{i,t}(a); w^g_*, w_{i,*})-f(x_{i,t}(a); \widehat{w}^g_t, \widehat{w}_{i,t}) \bigg\vert\right]   \nonumber \\
     &= \mathcal{O}\left(\left(\frac{\left(\|w_*^g\|^2 + \sum_{i=1}^P  \|w_{i,*}\|^2 \right)\sigma^2 B}{PT}\right)^{\frac{1}{4}} + \frac{\left(\gamma D^4 G^2(\tau+B)^2\right)^{\frac{1}{6}}}{T^\frac{1}{3}} + \left(\frac{DG(\tau + B)}{T}\right)^{\frac{1}{2}} \right). \label{eqn: CB tmp regret bound}
\end{align}
Now we consider the regret of the contextual bandit problem defined in \eqref{eqn: CB regret}. Notice that by defining $a_{i,t}^*=\argmax_a f(x_{i,t}(a); w^g_*, w_{i,*})$, we have
\begin{align}
     &f(x_{i,t}(a_{i,t}^*); w^g_*, w_{i,*})  - f(x_{i,t}(a_{i,t}); w^g_*, w_{i,*})   \nonumber  \\
     &\leq \Big\vert f(x_{i,t}(a_{i,t}^*); w^g_*, w_{i,*}) - f(x_{i,t}(a_{i,t}^*); \widehat{w}^g_t, \widehat{w}_{i,t}) \Big\vert   \nonumber  \\
     &\qquad + f(x_{i,t}(a_{i,t}^*); \widehat{w}^g_t, \widehat{w}_{i,t}) - f(x_{i,t}(a_{i,t}); \widehat{w}^g_t, \widehat{w}_{i,t}) \nonumber  \\
     &\qquad + \Big\vert f(x_{i,t}(a_{i,t}); \widehat{w}^g_t, \widehat{w}_{i,t}) - f(x_{i,t}(a_{i,t}); w^g_*, w_{i,*}) \Big\vert  \nonumber  \\
     &\leq \underbrace{f(x_{i,t}(a_{i,t}^*); \widehat{w}^g_t, \widehat{w}_{i,t}) - f(x_{i,t}(a_{i,t}); \widehat{w}^g_t, \widehat{w}_{i,t})}_{\term_1} + \underbrace{2\sum_{a=1}^K \Big\vert  f(x_{i,t}(a); w^g_*, w_{i,*}) -  f(x_{i,t}(a); \widehat{w}^g_t, \widehat{w}_{i,t}) \Big\vert}_{\term_2}   \nonumber    \label{eqn: CB regret decomposition}
\end{align}
By our strategy of choosing actions (Eq.\eqref{eqn: CB choose action}), when $t\neq nB$, $\term_1$ is non-positive. Besides, we can bound the sum of $\term_2$ using \eqref{eqn: CB tmp regret bound}. Thus combining everything we get
\begin{align*}
    &\E\left[\frac{1}{PT}\sum_{i=1}^P \sum_{t=1}^T \left(f(x_{i,t}(a_{i,t}^*); w^g_*, w_{i,*})  - f(x_{i,t}(a_{i,t}); w^g_*, w_{i,*})\right)\right]\\
    &\leq \frac{1}{B} + \E\left[\frac{2}{PT} \sum_{i=1}^P \sum_{t=1}^T \sum_{a=1}^K  \Big\vert  f(x_{i,t}(a); w^g_*, w_{i,*}) -  f(x_{i,t}(a); \widehat{w}^g_t, \widehat{w}_{i,t}) \Big\vert \right] \\
    &\leq \frac{1}{B} +  \mathcal{O}\left(K\left(\frac{\left(\|w_*^g\|^2 + \sum_{i=1}^P \|w_{i,*}\|^2\right)\sigma^2 B}{PT}\right)^{\frac{1}{4}}  + K\cdot \frac{\left(\gamma D^4 G^2(\tau+B)^2\right)^{\frac{1}{6}}}{T^\frac{1}{3}} + K\cdot\left(\frac{DG(\tau + B)}{T}\right)^{\frac{1}{2}} \right).
\end{align*}
Seting
\begin{align*}
     B = \min\left\{\left(\frac{PT}{K^4\left(\|w_*^g\|^2 + \sum_{i=1}^P \|w_{i,*}\|^2 \right)\sigma^2}\right)^{\frac{1}{5}},
          \frac{T^{\frac{1}{4}}}{\left(K^6 \gamma D^4 G^2\right)^{\frac{1}{8}}},
          \frac{T^{\frac{1}{3}}}{(K^2DG)^{\frac{1}{3}}} \right\},
\end{align*}
we get the bound of
\begin{align*}
     \mathcal{O}\left(\left(\frac{K^4\left(\|w_*^g\|^2 + \sum_{i=1}^P \|w_{i,*}\|^2 \right)\sigma^2}{PT}\right)^{\frac{1}{5}} + \frac{\left(K^6 \gamma D^4 G^2\right)^{\frac{1}{8}}}{T^{\frac{1}{4}}} + \frac{K(\gamma D^4 G^2 \tau^2)^{\frac{1}{6}}+ (K^2 DG)^{\frac{1}{3}}}{T^{\frac{1}{3}}} + \frac{K\sqrt{DG\tau}}{\sqrt{T}}\right).
\end{align*}
\end{proof}

\section{More Experimental Results}
\label{sec:exp-details}

\subsection{The effect of the number of clients with no delay}
\label{subsec: effect of number of clients}
In Section~\ref{sec: exp}, we showed the effect of the number of workers for four of the datasets we test on (see Figure~\ref{fig: the performance with number of clients group 1} and Figure~\ref{fig: the performance with number of clients group 2}) in the absence of delay. In Figure~\ref{fig: the performance with number of clients} we provide the plots for the other four datasets we use.

\begin{figure}[H]
\subfigure[satimage]{
   \includegraphics[width=8cm]{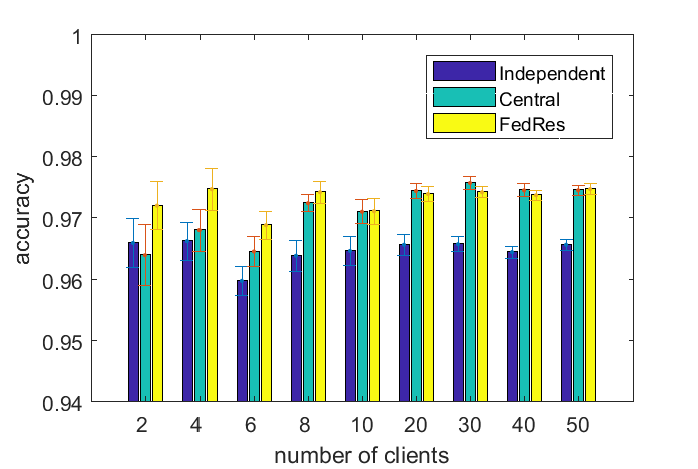}
 }\hfill
\subfigure[usps]{
   \includegraphics[width=8cm]{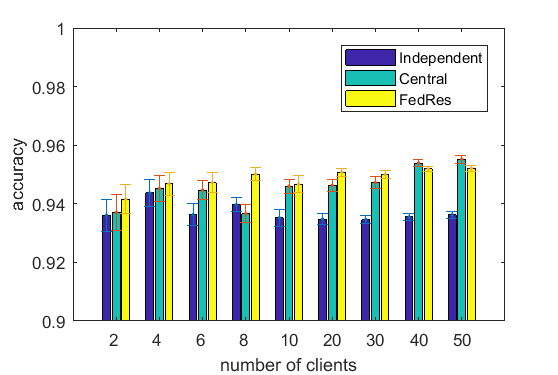}
 }\hfill
\subfigure[shuttle]{
   \includegraphics[width=8cm]{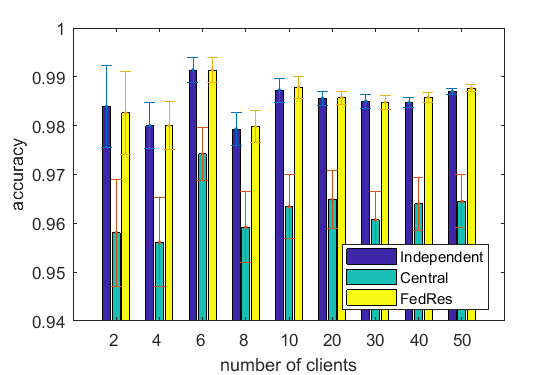}
 }\hfill
\subfigure[covtype]{
   \includegraphics[width=8cm]{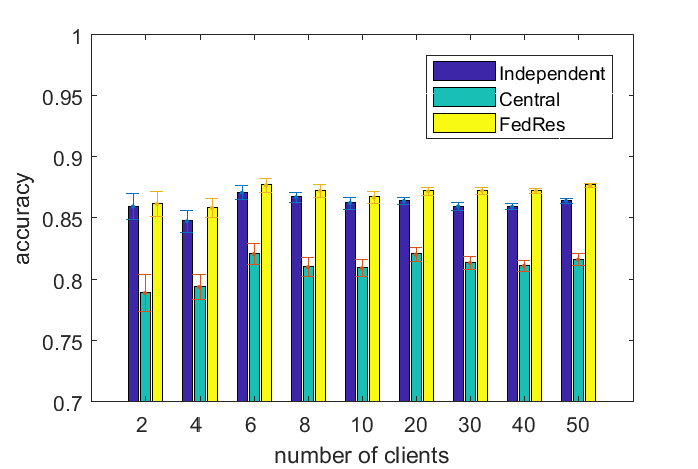}
 }\hfill

 \caption{Test accuracy versus the numbers of clients. In these experiments, we let the delay be zero. Each data point is an average over $50$ random trials. }
 \label{fig: the performance with number of clients}
\end{figure}

\subsection{The effect of delay}
\label{subsec: effect of delay in appendix}
In this section, we extend Section~\ref{subsec: effect of delay}, showing more experimental results to see the effect of delay on the performance of the algorithms. We compare the following three schemes:
\begin{enumerate}
    \item \textbf{Independent without delay}: same as the Independent scheme described in Section~\ref{subsec: test algorithms and implementation}
    \item \textbf{Central with delay}: same as the Central scheme described in Section~\ref{subsec: test algorithms and implementation}, but with delayed communication between the server and the clients.
    \item \textbf{\alg with delay}: same as the \alg described in Section~\ref{subsec: test algorithms and implementation}, but with delayed communication between the server and the clients.
\end{enumerate}
We make the above assumptions because for Central and \alg, there is communications between the server and the clients, while for Independent, all learning happens locally on clients. We plot the test accuracy for the case the number of clients is $50$ under different amount of delay ranging from $0$ to $200$ (for Independent, we simply plot a constant that corresponds to the accuracy without delay). Like in Section~\ref{subsec: experiment for robustness}, we separate the discussions for two types of datasets: those for which Central is better than Independent, and those Independent is better than Central.

\newpage
\paragraph{Type 1 datasets: Central is better than Independent (mnist, satimage, sensorless, usps)}
For this type of datasets, we see from Figure~\ref{figure: type 1 delay} that in three out of the four datasets (mnist, satimage, usps), \alg and Central are robust with delays, while \alg constantly outperform both baselines. For the sensorless dataset, \alg and Central suffer from degradation with delays, among which \alg has a somewhat worse degradation. However, \alg still outperforms Central when the delay is not excessively large.
\begin{figure}[H]
\subfigure[mnist]{
   \includegraphics[width=8cm]{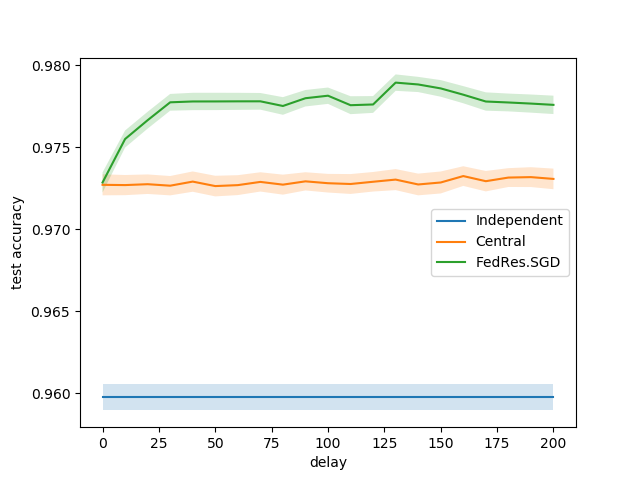}
 }\hfill
\subfigure[satimage]{
   \includegraphics[width=8cm]{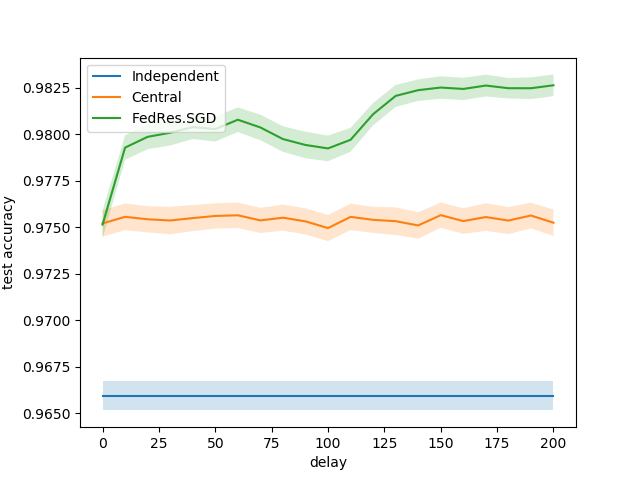}
 }\hfill
\subfigure[sensorless]{
   \includegraphics[width=8cm]{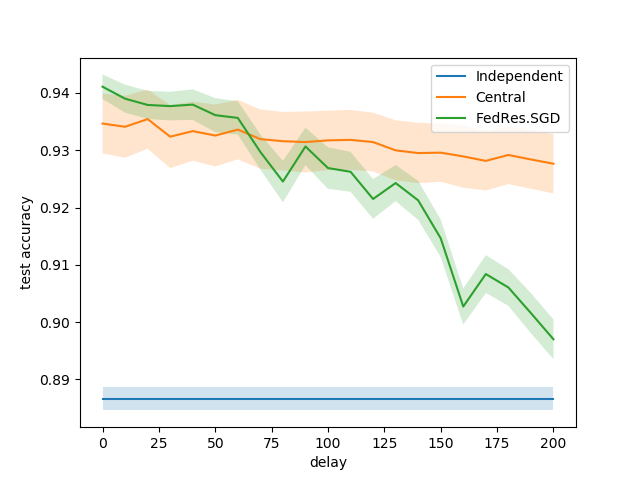}
 }\hfill
\subfigure[usps]{
   \includegraphics[width=8cm]{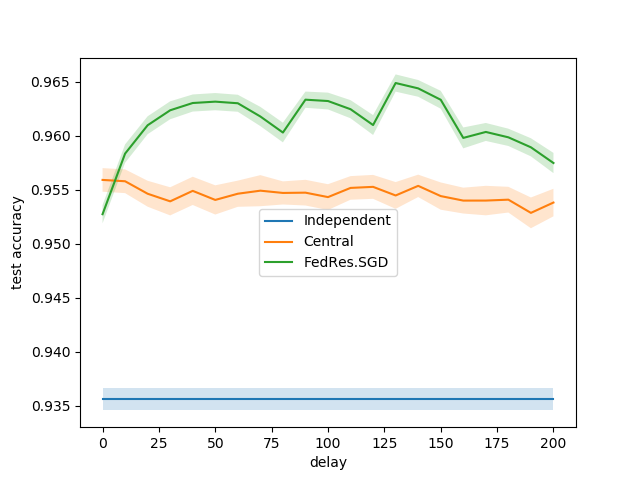}
 }\hfill

\caption{Test accuracy versus delay for mnist, satimage, sensorless, usps. We let the number of clients be $50$. Each data point is an average over $50$ random trials. }
\label{figure: type 1 delay}
\end{figure}

\newpage
\paragraph{Type 2 datasets: Independent is better than Central (letter, pendigits, shuttle, covtype)}
For this type of datasets, we already argued in Section~\ref{subsec: experiment for robustness} that the federated scheme does not provide clear advantages over the Independent baseline. As seen in Figure~\ref{figure: type 2 delay}, when coupled with delay, \alg can actually perform worse than Independent (letter, pendigits, shuttle) even when the delay is of moderate amount. This is likely due to a combination of these datasets not requiring too many samples to learn a good predictor so that the Independent scheme succeeds, and a lack of similarity in the prediction problems across clients which means that the shared global component does not accelerate learning significantly.
\begin{figure}[H]
\subfigure[letter]{
   \includegraphics[width=8cm]{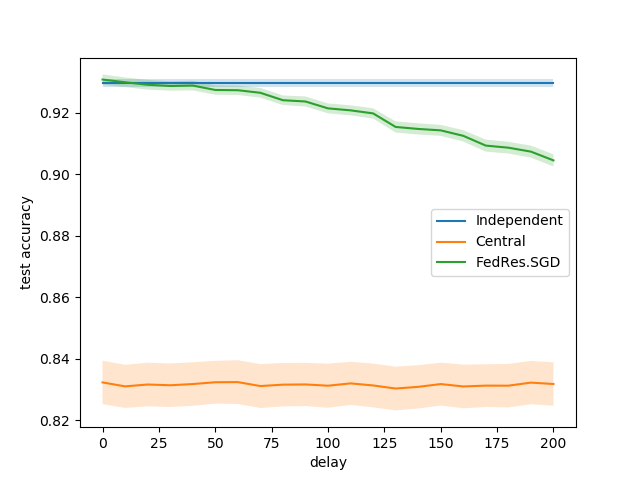}
 }\hfill
\subfigure[pendigits]{
   \includegraphics[width=8cm]{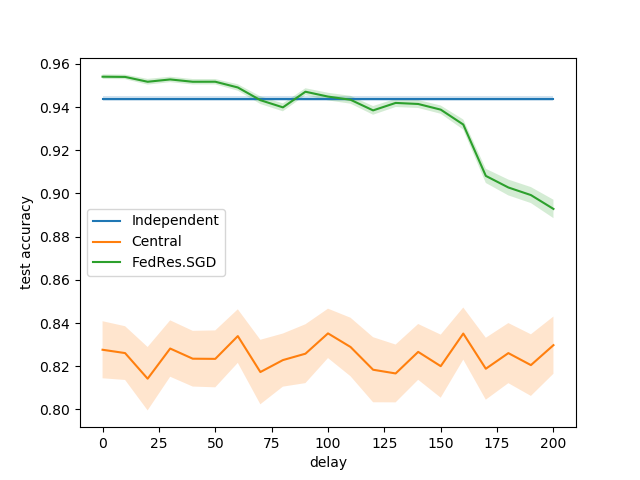}
 }\hfill
\subfigure[shuttle]{
   \includegraphics[width=8cm]{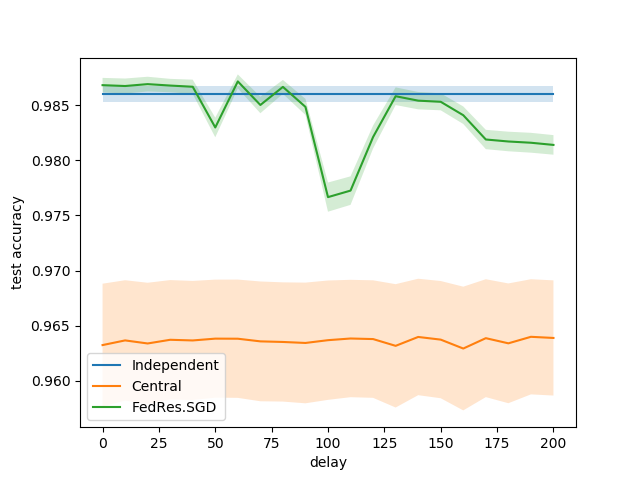}
 }\hfill
\subfigure[covtype]{
   \includegraphics[width=8cm]{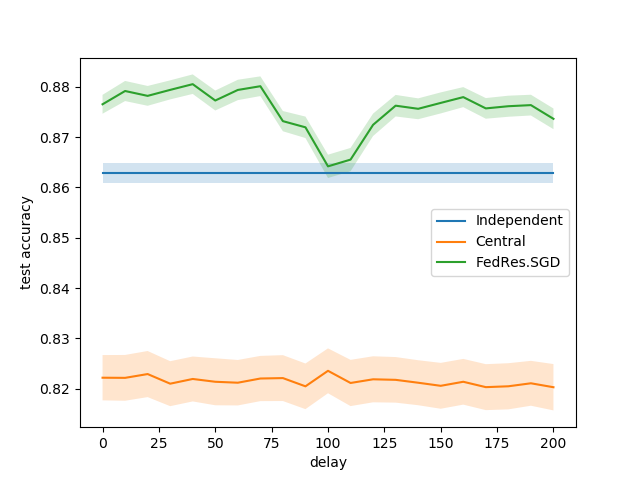}
 }\hfill

\caption{Test accuracy versus delay for letter, pendigits, shuttle, covtype. We let the number of clients be $50$. Each data point is an average over $50$ random trials.}
\label{figure: type 2 delay}
\end{figure}

\subsection{Test accuracy versus the number of clients with delay}
\label{subsec: combined plots}
In this subsection, we provide plots of ``test accuracy versus the number of clients'' under delay (i.e., similar to Figure~\ref{fig: the performance with number of clients} but with delay). As explained in Section~\ref{subsec: effect of delay in appendix}, the we only apply delay on the Central and \alg schemes, but not on Independent. We plot the cases for delay being $20$ and $80$ in Figure~\ref{fig: delay=20 case} and \ref{fig: delay=80 case} respectively. We observe similar patterns in the two figures, with the performance loss being larger for the higher delay, though it is typically overcome as the number of clients increases.

\begin{figure}[H]
\subfigure[mnist]{
   \includegraphics[width=7cm]{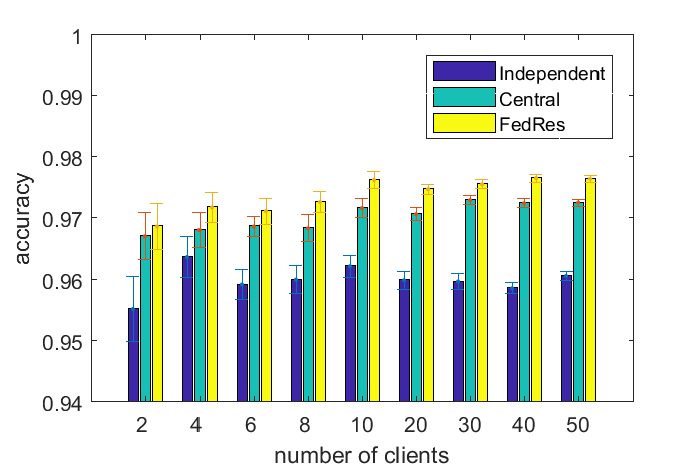}
 }\hfill
\subfigure[satimage]{
   \includegraphics[width=7cm]{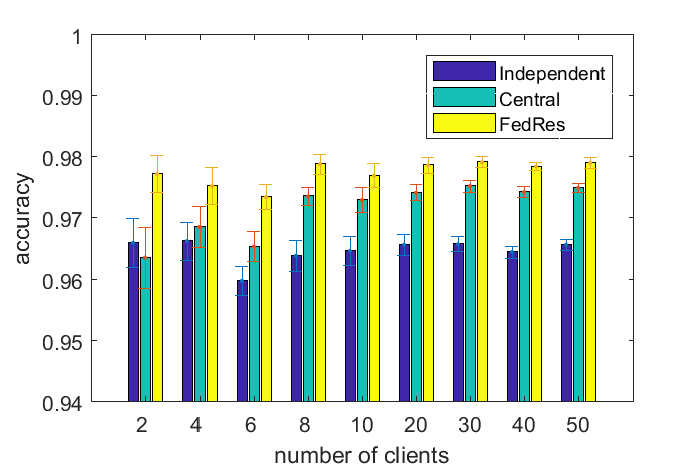}
 }\hfill
\subfigure[sensorless]{
   \includegraphics[width=7cm]{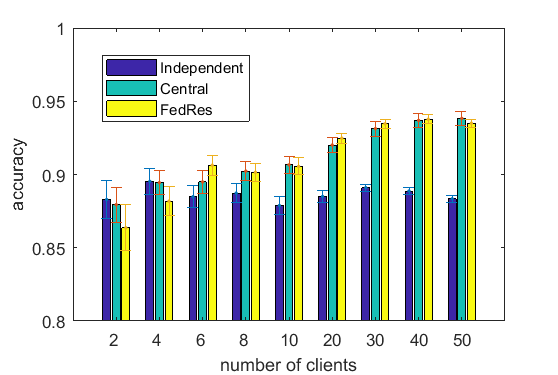}
 }\hfill
\subfigure[usps]{
   \includegraphics[width=7cm]{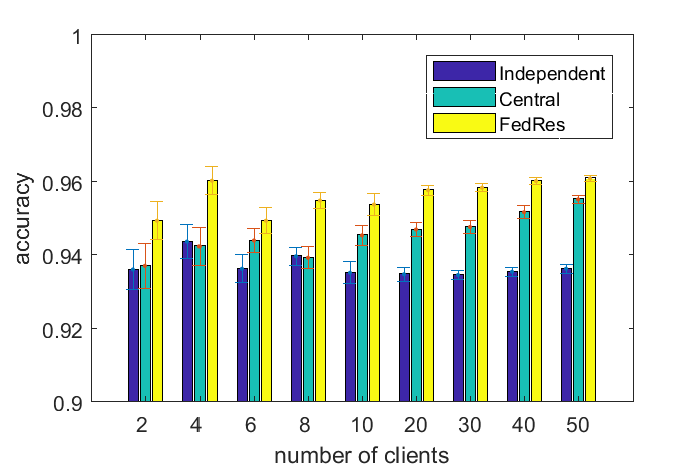}
 }\hfill
\subfigure[letter]{
   \includegraphics[width=7cm]{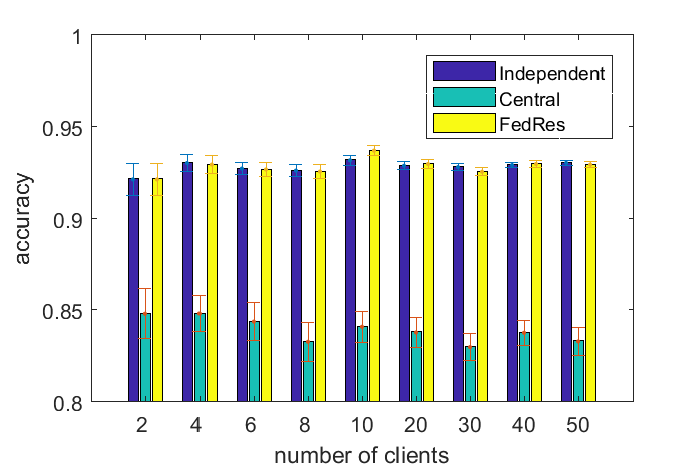}
 }\hfill
\subfigure[pendigits]{
   \includegraphics[width=7cm]{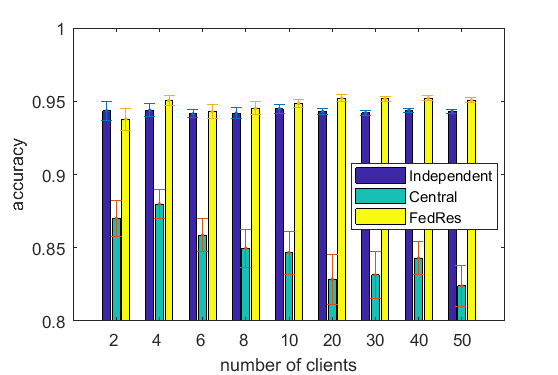}
 }\hfill
\subfigure[shuttle]{
   \includegraphics[width=7cm]{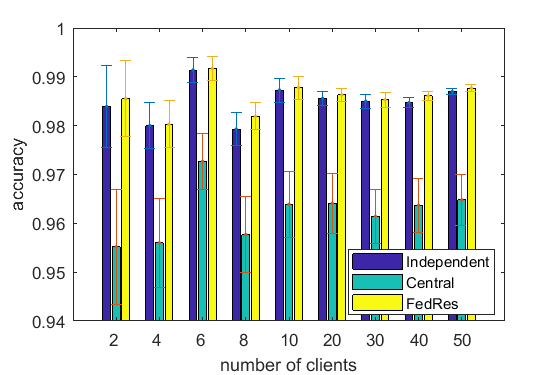}
 }\hfill
\subfigure[covtype]{
   \includegraphics[width=7cm]{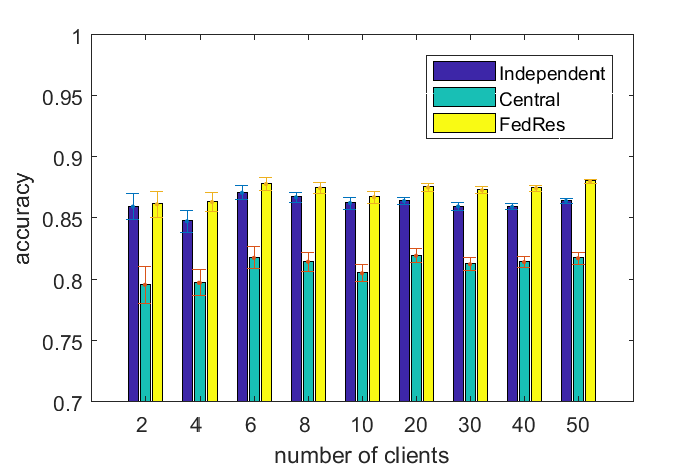}
 }\hfill

\caption{Test accuracy versus the number of clients under fixed delay of $20$. Each data point is an average over $50$ random trials.}
\label{fig: delay=20 case}
\end{figure}

\newpage

\begin{figure}[H]
\subfigure[mnist]{
   \includegraphics[width=7cm]{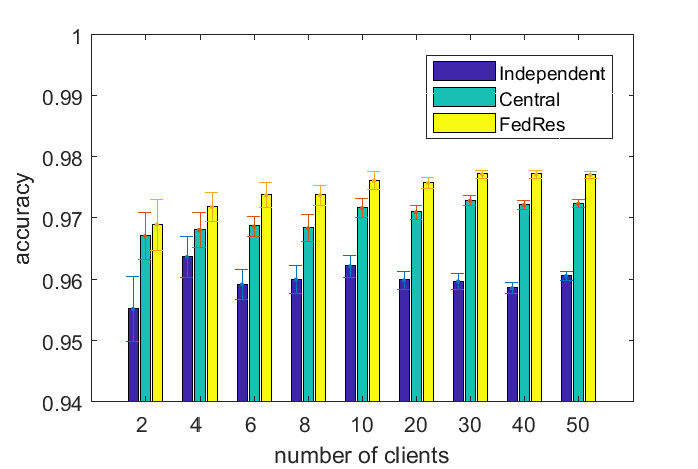}
 }\hfill
\subfigure[satimage]{
   \includegraphics[width=7cm]{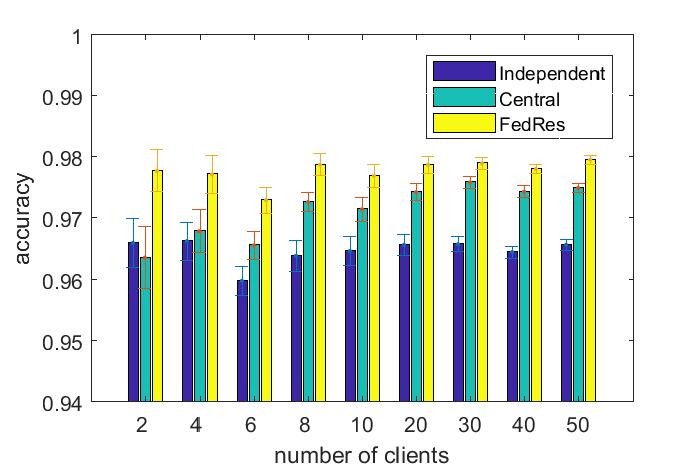}
 }\hfill
\subfigure[sensorless]{
   \includegraphics[width=7cm]{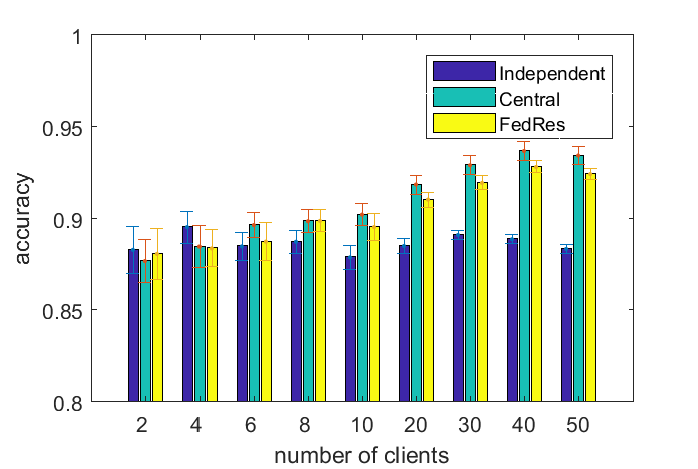}
 }\hfill
\subfigure[usps]{
   \includegraphics[width=7cm]{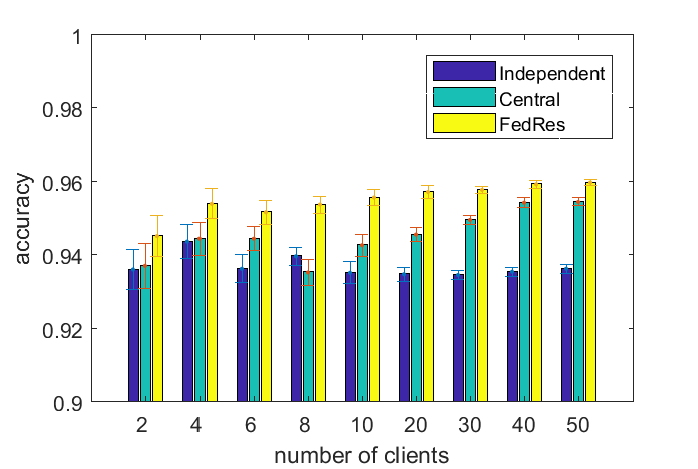}
 }\hfill
\subfigure[letter]{
   \includegraphics[width=7cm]{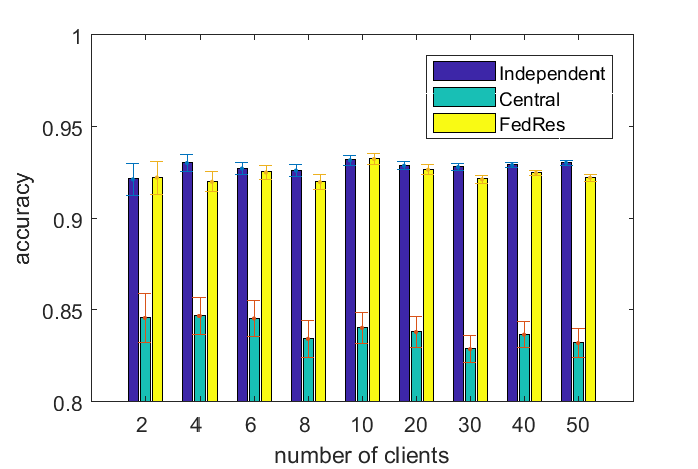}
 }\hfill
\subfigure[pendigits]{
   \includegraphics[width=7cm]{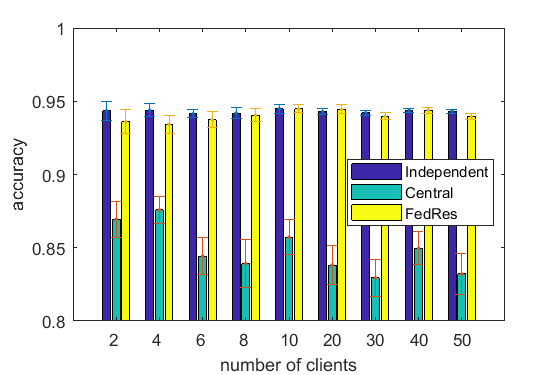}
 }\hfill
\subfigure[shuttle]{
   \includegraphics[width=7cm]{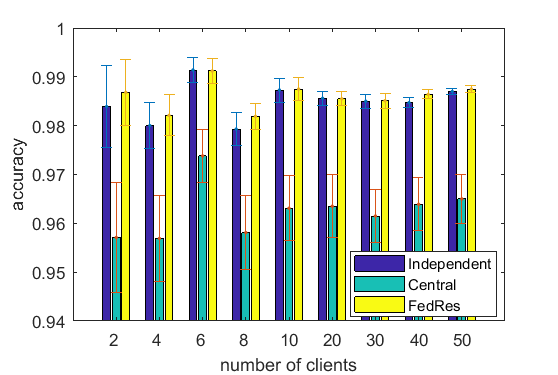}
 }\hfill
\subfigure[covtype]{
   \includegraphics[width=7cm]{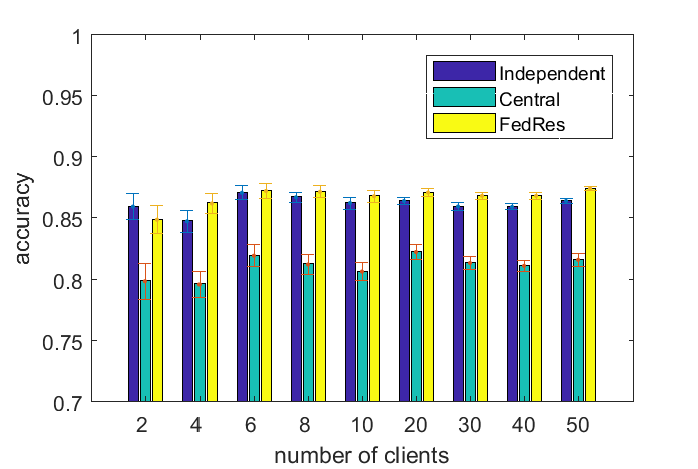}
 }\hfill

\caption{Test accuracy versus the number of clients under fixed delay of $80$. Each data point is an average over $50$ random trials.}
\label{fig: delay=80 case}
\end{figure}

\subsection{Conclusions from the experiments}
\label{app exp conclusion}
From the discussion and the experimental results in Section~\ref{subsec: experiment for robustness}, when there is no delay (or insignificant delay), the \alg provides robustness to the task similarity among clients --- it takes advantage of the equivalently larger datasets when the tasks of the clients are similar, and keeps the performance similar to Independent when the Central scheme is actually harmful. From the extensive experiments shown in Section~\ref{subsec: effect of delay in appendix} and Section~\ref{subsec: combined plots}, we see that when delay is presented, \alg is generally robust despite these delays, particularly when the clients can jointly learn a good global model. On the other hand, in settings where each client can learn a reasonably good model locally, the delay can be more harmful. In general, we find that \alg presents a robust way of leveraging shared learning when it is helpful, while competing well with completely local learning when that is the best thing to do, even in the face of communication delays.

\end{document}